\documentclass[journal]{IEEEtran}

\usepackage{graphicx}
\usepackage{amsthm}
\usepackage{amssymb,latexsym,multicol,caption}     % Standard packages
\usepackage{algorithm}  
\usepackage{algorithmic}
\usepackage{epsfig} 
\usepackage{epstopdf}
\usepackage[colorlinks, linkcolor=blue, anchorcolor=red, citecolor=red, urlcolor=black]{hyperref}
\usepackage{color}
\usepackage{multirow}
\usepackage{multicol}
\usepackage{subfigure}
\usepackage{placeins}
\usepackage{bm}
\usepackage{graphics}
\usepackage{footnote}
\usepackage{url}
\usepackage{indentfirst}
\usepackage{longtable}
\usepackage{array}
\usepackage{mdwmath}
\usepackage{mdwtab}
\usepackage{rotating}
\usepackage{color}
\usepackage{morefloats}
\usepackage{slashbox}
\usepackage{cite}
\usepackage{mcite}
\usepackage{eqnarray}
\usepackage{float}
\usepackage{makecell}
\usepackage{amsmath}
\usepackage{threeparttable}
\usepackage{pifont}
\usepackage{cleveref}
\usepackage{balance}
\usepackage{flushend}

\renewcommand{\eqref}[1]{(\ref{#1})}

\newcommand{\argmin}{\arg \min}

\crefname{equation}{}{}
\crefname{figure}{Fig.}{Figs.}
\crefname{table}{Table}{Tables}
\crefname{section}{Section}{Sections}
\crefname{prop}{Proposition}{Propositions}
\crefname{theorem}{Theorem}{Theorems}
\crefname{lemma}{Lemma}{Lemmas}
\newtheorem{theorem}{Theorem}[section]

\newtheorem{prop}{Proposition}[section]
\newtheorem{definition}{Definition}[section]

\newcommand{\tr}[1]{\textbf{Tr}{(#1)}}

\crefname{figure}{Fig.}{Figs.}
\crefname{table}{Table}{Tables}
\crefname{section}{Section}{Sections}
\crefname{lemma}{Lemma}{Lemmas}
\crefname{definition}{Definition}{Definitions}
\crefname{theorem}{Theorem}{Theorems}
\crefname{prop}{Proposition}{Propositions}

\theoremstyle{plain}

\theoremstyle{definition}

\newtheorem{remark}{Remark}

\newcommand{\A}{KW}
\newcommand{\B}{D_KW}
\newcommand{\C}{W^TKW}

\newcommand{\Ab}{KG}
\newcommand{\Bb}{D_KG}
\newcommand{\Cb}{K}
\newcommand{\Db}{G^TG}

\ifCLASSINFOpdf
\else
\fi

\begin{document}

\title{ Nonnegative Matrix Factorization with Local Similarity Learning }
%

%\author{Michael~Shell,~\IEEEmembership{Member,~IEEE,}
%        John~Doe,~\IEEEmembership{Fellow,~OSA,}
%        and~Jane~Doe,~\IEEEmembership{Life~Fellow,~IEEE}% <-this % stops a space
%\thanks{M. Shell was with the Department
%of Electrical and Computer Engineering, Georgia Institute of Technology, Atlanta,

\author{Chong~Peng,
        Chenglizhao~Chen, Zhao~Kang, 
        and~Qiang~Cheng,~\IEEEmembership{Senior~Member,~IEEE}% <-this % stops a space
\thanks{C. Peng and C. Chen are with the College of Computer Science and Technology, Qingdao University, Qingdao, Shandong, 266000, China; Z. Kang is with School of Computer Science and Engineering, University of Electronic Science and Technology
of China, Chengdu, Sichuan 611731, China; Q. Cheng is with Institute of Biomedical Informatics \& Department of Computer Science, University of Kentucky, Lexington, KY 40536, USA. E-mail: (pchong1991@163.com, cclz123@163.com, sckangz@gmail.com, qiang.cheng@uky.edu).}
        
%\thanks{C. Peng, Z. Kang, and Q. Cheng are with the Department of Computer Science, Southern Illinois University, Carbondale, IL, 62901 USA. E-mail: (pchong@siu.edu, zhao.kang@siu.edu, qcheng@cs.siu.edu).}% <-this % stops a space
}



\markboth{Submitted to xxxx,~Vol.~00, No.~00, December~2018}%
{Chong Peng \MakeLowercase{\textit{et al.}}: Bare Demo of IEEEtran.cls for IEEE Journals}




\maketitle

\begin{abstract}
Existing nonnegative matrix factorization methods focus on learning global structure of the data to construct basis and coefficient matrices, which ignores the local structure that commonly exists among data. In this paper, we propose a new type of nonnegative matrix factorization method, which learns local similarity and clustering in a mutually enhancing way. The learned new representation is more representative in that it better reveals inherent geometric property of the data. Nonlinear expansion is given and efficient multiplicative updates are developed with theoretical convergence guarantees. Extensive experimental results have confirmed the effectiveness of the proposed model.
\end{abstract}


% \fi
\IEEEpeerreviewmaketitle

\section{Introduction}
High-dimensional data are ubiquitous in the learning community and it has become increasingly challenging to learn from such data \cite{duda2012pattern}. For example, as one of the most important tasks in, for example, multimedia and data mining, information retrieval has drawn considerable attentions in recent years \cite{zhu2018low,Kang2015Learning,Zhen2015Spectral}, where there is often a need to handle high-dimensional data. Often times, it is desirable and demanding to seek a data representaiton to reveal latent data structures of high-dimensional data, which is usually helpful for further data processing. It is thus a critical problem to find a suitable representation of the data \cite{cai2008non,lee1999learning,li2006relationships,tao2007general} in many learning tasks, such as single image super-resolution \cite{Zhu2014Fast}, image reconstruction \cite{Ogawa2011Missing}, image clustering \cite{Peng2017Subspace}, foreground-background seperation in surveillance video \cite{candes2011robust}, matrix completion \cite{nie2012low}, etc. To this end, a number of methods for finding proper representations have been developed, among which matrix factorization technique has been widely used to handle high-dimensional data. Matrix factorization seeks two or more low-dimensional matrices to approximate the original data such that the high-dimensional data can be represented with reduced dimensions \cite{liu2013robust,peng2015subspace}. 

For some types of data, such as images and documents that are widely used in real world learning problems, the entries are naturally nonnegative. For such data, nonnegative matrix factorization (NMF) was proposed to seek two nonnegative factor matrices for approximation. 
%NMF has been observed to be successful in various applications, including media quality assesment \cite{Xia2017Media}, audiovisual document structuring \cite{Essid2013Smooth}, concept detection in videos \cite{Tang2012Sparse}, image retrieval \cite{Ma2017Learning}, temporal psychovisual modulation display \cite{Gao2016Factorization}, etc. 
In fact, the way of seeking nonnegative factorization for nonnegative data naturally leads to learning parts-based representations of the data \cite{lee1999learning}. Parts-based representation is believed to commonly exist in human brain with psychological and physiological evidence \cite{palmer1977hierarchical,wachsmuth1994recognition,logothetis1996visual}. It overcomes the drawback of latent semantic indexing (LSI) \cite{deerwester1990indexing}, for which the interpretation of basis vectors is difficult due to mixed signs. When the number of basis vectors is large, NMF has been proven to be NP-hard \cite{vavasis2009complexity}; moreover, \cite{Arora2012Computing} has recently given some conditions, under which NMF is solvable. Recent studies have shown a close relationship between NMF and K-means \cite{ding2005equivalence}, and further study has shown that both spectral clustering and kernel K-means \cite{dhillon2007weighted} are particular cases of clustering with NMF under a doubly stochastic constraint \cite{zass2005unifying}. This implies that NMF is especially suitable for clustering such data. In this paper, we will develop a novel NMF method, which focuses on the clustering capability.

Many variants of NMF have been developed in the past decades, which can be mainly categorized into four types, including basic NMF \cite{lee1999learning}, constrained NMF \cite{ding2006orthogonal}, structured NMF \cite{yoo2010orthogonal}, and generalized NMF \cite{buciu2008nonnegative}. A fairly comprehensive review can be found in \cite{wang2013nonnegative}. Among these methods, Semi-NMF \cite{ding2010convex} removes the nonnegative constraint on the data and basis vectors, such that its applications can be expanded to more fields; convex NMF (CNMF) \cite{ding2010convex} restricts the basis vectors to lie in the feature space of the input data so that they can be represented as convex combinations of data vectors; orthogonal NMF (ONMF) \cite{ding2006orthogonal} imposes orthogonality constraints on factor matrices, which leads to clustering interpretation. The classic NMF only considers the linear structures of the data by finding new data points with respect to the new basis and ignores the nonlinear structures of the data, which is usually important for many applications such as clustering. To learn the latent nonlinear structures of the data, graph regularized nonnegative matrix factorization (GNMF) considers the intrinsic geometrical structures of the data on a manifold by incorporating a Laplacian regularization \cite{cai2011graph}. By modeling the data space as a manifold embedded in an ambient space and performing NMF on this manifold, GNMF considers both linear and nonlinear relationships of the data points in the original instance space, and thus it is also more discriminating than ordinary NMF which only considers the Euclidean structure of the data \cite{cai2011graph}. This renders GNMF more suitable for clustering purpose than the original NMF. Based on GNMF, robust manifold nonnegative matrix factorization (RMNMF) constructs a structured sparsity-inducing norm-based robust formulation \cite{huang2014robust}. With a $\ell_{2,1}$-norm, RMNMF is insensitive to the between-sample data outliers and improves the robustness of NMF \cite{huang2014robust}. Moreover, the relaxed requirement on signs of the data makes it a nonlinear version of Semi-NMF.

% graph regularized NMF (GNMF) \cite{cai2011graph} incorporates graph Laplacian to learn nonlinear relationships of the data on manifold; robust manifold NMF (RMNMF) \cite{huang2014robust} adopts robust $\ell_{2,1}$ norm as a loss function to measure fitting errors, and the relaxed requirement on signs of the data makes it a nonlinear version of Semi-NMF.

In recent years, the importance of preserving local manifold structure has drawn considerable attentions in research community of machine learning, data mining, and pattern recognition \cite{zhao2017robust,nie2014clustering,liu2014global,chen2013local}. It has been shown that besides pairwise sample similarity, local geometric structure of the data is also crucial in revealing underlying structure of the data \cite{liu2014global}: 1)In the transformed low-dimensional space, it is important to maintain the intrinsic information of high-dimensional data \cite{Wang2008Locality}; 2) It may be insufficient to represent the underlying structures of the data with a single characterization and both global and local ones are necessary \cite{Chen2007Integrating}; 3) In some ways, we can regard the local geometric structure of the data as data dependent regularization, which helps avoid overfitting issues \cite{liu2014global}. Despite its importance, local structure of data has yet to be exploited in NMF study. In this paper, we propose a new type of NMF method, which simultaneously learns both similarity and geometric/clustering structures of the data and clustering such that the learned basis and coefficients well preserve discriminative information of the data. Recent studies reveal that high-dimensional data often reside in a union of low-dimensional subspaces and the data can be self-expressed by a low-dimensional representation \cite{liu2013robust,elhamifar2013sparse}, which can be regarded as pairwise similarity of samples. Instead of simply using pairwise similarity of samples, in our method, we transform the pairwise similarity into the similarity between a score vector of a sample on basis and the representation of another sample in the same cluster, which integrates basis and coefficient learning into simultaneous similarity learning and clustering. Nonlinear model is developed to measure both local and global nonlinear relationships of the data. 

The main contributions of this paper are as follows: 
\begin{itemize}
\item For the first time, in an effective yet simple way, local similarity learning is embedded into learning matrix factorization, which allows our method to learn global and local structures of the data. The learned basis and representations well preserve the inherent structures of the data and are more representative; 
\item To our best knowledge, we are the first to integrate the orthogonality-constrained coefficient matrix into local similarity adaption, such that local similarity and clustering can mutually enhance each other and be learned simultaneously; 
\item Nonlinear extension is developed from kernel perspectives, which can be further expanded to cope with multiple-kernel scenario; 
\item Efficient multiplicative update rules are constructed to solve the proposed model and comprehensive theoretical analysis is provided to guarantee the convergence; 
\item Lastly, extensive experimental results have verified the effectiveness of our method.
\end{itemize}

The rest of this paper is organized as follows: In \cref{sec_related}, we briefly review some methods that are closely related with our research. Then we introduce our method in \cref{sec_proposed}. Regarding the proposed method, we provide an efficient alternating optimization procedure in \cref{sec_optimization}, and then provide complicated theoretical results for the convergence analysis in \cref{sec_proof}. Next, we conduct comprehensive experiments and show the results in \cref{sec_experiments}. Finally, we conclude the paper in \cref{sec_conclusion}.

%
%\begin{itemize}
%\item We introduce a local similarity learning regularization term to NMF framework, which allows our model simultaneously learn both global and local structures of the data. In our model, the similarity is translated to the similarity of two vectors, one of which is the score vector of an sample on the basis while another is the representation vector, by learning which both global and local structures are well preserved in the basis and coefficients;
%\item We enforce orthogonality constraint on coefficient matrix, which provides better interpretation of clustering;
%\item Nonlinear extensions are developed from different aspects, which can further be expanded to multiple kernel or manifold learning scenarios;
%\item Efficient multiplicative update rules are designed to solve the proposed model with theoretical analysis provided;
%\item Extensive experimental results have verified the effectiveness of our method.
%\end{itemize}

%The rest of this paper is organized as follows. We review some related work in \cref{sec_related}. Then we develop our model in \cref{sec_proposed}. The optimization and theoretical analysis are presented in \cref{sec_optimization,sec_proof}, respectively. Experimental results are presented in \cref{sec_experiments}. Finally, \cref{sec_conclusion} concludes the paper.

\textbf{Notation:} For a matrix $M$, $M_{ij}$, $M_i$, and $M_{\bar{j}}$ denote the $ij$-th element, $i$-th column, and $j$-th row of $M$. $\tr{\cdot}$ is the trace operator, $\|\cdot\|_F$ and $\|\cdot\|_2$ are the Frobenius and $\ell_2$ norms. $I_k$ denotes the identity matrix of size $k \times k$, $\text{diag}(\cdot)$ is an operator that returns a diagonal matrix with identical diagonal elements to the input matrix.

\section{Related Work}
\label{sec_related}
In this section, we briefly review some methods that are closely related with our research.
\subsection{NMF}
Given nonnegative data $X = [x_1,\cdots,x_n]\in\mathcal{R}^{p\times n}$ with $p$ being the dimension and $n$ sample size, NMF is to factor $X$ into $U\in\mathcal{R}^{p\times k}$ (basis) and $G\in\mathcal{R}^{n\times k}$ (coefficients) with the following optimization problem:
\begin{equation}
\label{eq_nmf}
%\min_{U\ge 0,G\ge 0} \|X - UG^T \|_F^2,
\min_{U\ge 0,G\ge 0} \|X - UG^T \|_F^2,
\end{equation}
where $k\ll n$ enforces a low-rank approximation of the original data. %Here, $U$ are the basis vectors and $V$ is the coefficient matrix, which are new representations of the data with respect to the basis vectors.

%\subsection{Graph Laplacian}
%\label{sec_related_gnmf}
%Graph Laplacian \cite{chung1997spectral} is widely used to incorporate the geometrical structure of the data on manifold. In particular, the manifold enforces the smoothness of the data in linear and nonlinear spaces by minimizing
%\begin{small}
%\begin{equation}
%\label{eq_manifold}
%\begin{aligned}
%		&	\frac{1}{2}\sum_{i=1}^{n}\sum_{j=1}^{n} \|G_i - G_j\|_2^2 W^x_{ij}  \\ %\end{aligned}\end{equation}\begin{equation}\begin{aligned}\nonumber
%	=	& \sum_{j=1}^{n} D^x_{jj} G_j^T G_j - \sum_{i=1}^{n}\sum_{j=1}^{n} W^x_{ij} G_i^T G_j,	\\ 
%	=	& \textbf{Tr}(G^T D^x G) - \textbf{Tr}(G^T W^x G)	= \textbf{Tr}(G^T L^x G),
%\end{aligned}
%\end{equation}
%\end{small}
%$W^x$ is the weight matrix that measures the pair-wise similarities of original data points, $D^x$ is a diagonal matrix with $D^x_{ii} = \sum_{j}W^x_{ij}$, and $L^x=D^x-W^x$. By minimizing \eqref{eq_manifold}, we can have an effect that if two data points are close in the intrinsic geometry of the data distribution, then their new representations with respect to the new basis, $G_i$ and $G_j$, are also close \cite{cai2011graph}.

\subsection{Graph Laplacian}
\label{sec_related_gnmf}
Graph Laplacian \cite{chung1997spectral} is defined as
\begin{equation}
\label{eq_manifold}
\begin{aligned}
		&	\frac{1}{2}\sum_{i=1}^{n}\sum_{j=1}^{n} \|G_i - G_j\|_2^2 W^x_{ij}  \\ %\end{aligned}\end{equation}\begin{equation}\begin{aligned}\nonumber
	=	& \sum_{j=1}^{n} D^x_{jj} G_j^T G_j - \sum_{i=1}^{n}\sum_{j=1}^{n} W^x_{ij} G_i^T G_j,	\\ 
	=	& \textbf{Tr}(G^T D^x G) - \textbf{Tr}(G^T W^x G)	= \textbf{Tr}(G^T L^x G),
\end{aligned}
\end{equation}
where $W^x$ is the weight matrix that measures the pair-wise similarities of original data points, $D^x$ is a diagonal matrix with $D^x_{ii} = \sum_{j}W^x_{ij}$, and $L^x=D^x-W^x$. It is widely used to incorporate the geometrical structure of the data on manifold. In particular, the manifold enforces the smoothness of the data in linear and nonlinear spaces by minimizing \eqref{eq_manifold}, which leads to an effect that if two data points are close in the intrinsic geometry of the data distribution, then their new representations with respect to the new basis, $G_i$ and $G_j$, are also close \cite{cai2011graph}. This is closely related with spectral clustering (SC) \cite{shi2000normalized,ng2002spectral} and its further development \cite{nie2011spectral,nie2016constrained}. 

%Recently, a parameter-free method to construct graph Laplacian has been proposed in
%Nie et al. [2016], which allows the graph itself to be adjusted as part of the clustering
%procedure with an exact number (which can be set to be the number of clusters) of
%connected components. A

\section{Proposed Method}
\label{sec_proposed}
As aforementioned, existing NMF methods do not fully exploit local geometric structures, nor do they exploit close interaction between local similarity and clustering. In this section, we will propose an effective, yet simple, new method to overcome these two drawbacks. 

CNMF restricts the basis of NMF to convex combinations of the columns of the data, i.e., $U=XW$, which gives rise to the following:
\begin{equation}
\label{eq_convexnmf}
\min_{W\ge 0, G\ge 0} \|X - XWG^T\|_F^2.
\end{equation}
By restricting $U=XW$, \eqref{eq_convexnmf} has the advantage that it could interpret the columns of $U$ as weighted sums of certain data points and these columns correspond to centroids \cite{ding2010convex}. It is natural to see that $W_{ij}$ reveals the importance of basis $U_j$ to $x_i$ by $W_{ij}$.

%To allow for immediate interpretation of clustering from the coefficient matrix, we impose an orthogonality constraint of $G$, i.e., $G^TG = I_k$, leading to 
%\begin{equation}
%\label{eq_obj_1}
%\min_{W\ge 0, G\ge 0, G^TG=I_k} \|X - XWG^T\|_F^2.
%\end{equation}
%Noted that by enforcing $G^TG = I_k$, the problem of NMF is directly connected with clustering in that $G$ can be regarded as relaxed cluster indicators. 

It is noted that \eqref{eq_convexnmf} is closely related to subspace clustering \cite{liu2013robust,elhamifar2013sparse}. The observation is that high-dimensional data usually reside in low-dimensional subspaces and recovering such subspaces usually needs a self-expressiveness assumption, which refer to that the data can be approximately self-expressed as $X\approx XZ$ with a representation matrix $Z$. Local structures of the data are shown to be important \cite{nie2014clustering} and it is necessary to take into consideration local similarity in learning tasks. A natural assumption is that if two data points $x_i$ and $x_j$ are close to each other, then their similarity, $Z_{ij}$, should be large; otherwise, $Z_{ij}$ small. This assumption leads to the following minimization:
\begin{equation}
\label{eq_local}
\min_{Z} \sum_{ij} \|x_i - x_j\|_2^2 Z_{ij} \Leftrightarrow \min_{Z} \tr{ Z^T D},
\end{equation}
where $$D_{ij} = \|x_i - x_j\|_2^2,$$ or in matrix form, $$D = \textbf{1}_n \textbf{1}_n^T \text{diag}(X^TX) + \text{diag}(X^TX) \textbf{1}_n \textbf{1}_n^T - 2 X^TX,$$ with $\textbf{1}_n$ being a length-$n$ vector of 1s. It is noted that the minimization of \cref{eq_local} directly enforces $Z_{ij}$ to reflect the pair-wise similarity information of the examples. Noticing that $W$ and $G$ are nonnegative and inspired by self-expressiveness assumption, we take $WG^T$ as the similarity matrix $Z$, such that $Z_{ij} = W_{\bar{i}}G_{\bar{j}}^T$. Here, $W_{\bar{i}}$ is the \textbf{\emph{score vector}} of example $x_i$ on the basis vectors, and $G_{\bar{j}}$ is the \textbf{\emph{coefficient vector}} of the $j$-th sample with respect to the new basis. If $x_i$ and $x_j$ are close on data manifold or grouped into the same cluster, then it is natural that $W_{\bar{i}}$ and $G_{\bar{j}}$ have higher similarity; vice versa. This close relationship between the geometry of $x_i$ and $x_j$ on data manifold and the similarity of $W_{\bar{i}}$ and $G_{\bar{i}}$ suggests that using $WG^T$ as $Z$ in \eqref{eq_local} is indeed meaningful. To encourage the interaction between similarity learning and clustering, we incorporate \eqref{eq_local} into \eqref{eq_convexnmf} with $Z = WG^T$, obtaining the Local Similarity NMF (LS-NMF):
\begin{equation}
\label{eq_obj_lsnmf}
\begin{aligned}
&\min_{W,G} \frac{1}{2}\|X - XWG^T\|_F^2 + \lambda \tr{W^T D G},	\\
& s.t. \quad W\ge 0, G\ge 0.
\end{aligned}
\end{equation}
where $\lambda \ge 0$ is a balancing parameter. Now, it is seen that the first term in above model captures global structure of the data by exploiting linear representation of each example with respect to the overall data, while the second term exploits local structure of the data by the connection between local geometric structure and pairwise similarity.

To allow for immediate interpretation of clustering from the coefficient matrix, we impose an orthogonality constraint of $G$, i.e., $G^TG = I_k$, leading to 
\begin{equation}
\label{eq_obj_lsnmf_orth}
\begin{aligned}
& \min_{W,G} \frac{1}{2}\|X - XWG^T\|_F^2+ \lambda \tr{W^T D G},	\\
& s.t. \quad W\ge 0, G\ge 0, G^TG=I_k.
\end{aligned}
\end{equation}
Note that by enforcing $G^TG = I_k$, the problem of NMF is directly connected with clustering in that $G$ can be regarded as relaxed cluster indicators. More importantly, learning similarity and clustering are connected through such a $G$ matrix and can be mutually promoted through an iterative optimization process. At the end of the iteration, the optimized clustering results are directly given by $G$.

%Model \eqref{eq_obj_lsnmf_orth} only learns linear relationships of the data and omits the nonlinear ones, which usually exist and are important. To take nonlinear relationships of the data into consideration, two approaches can be used, of which one is to map the data points to a reproducing kernel Hilbert space while the other is to exploit nonlinear relationships of the data on manifold. 

%Model \eqref{eq_obj_lsnmf_orth} only learns linear relationships of the data and omits the nonlinear ones, which usually exist and are important. To take nonlinear relationships of the data into consideration, we adopt the kernel method.  two approaches can be used, of which one is to map the data points to a reproducing kernel Hilbert space while the other is to exploit nonlinear relationships of the data on manifold. 

Model \eqref{eq_obj_lsnmf_orth} only learns linear relationships of the data and omits the nonlinear ones, which usually exist and are important. To take nonlinear relationships of the data into consideration, it is widely considered to seek data relationships in kernel space. 

We define a kernel mapping as $\phi:\mathcal{R}^{p}\rightarrow \mathcal{R}^{\bar{p}}$, which maps the data points $x_i\in\mathcal{R}^{p}$ from the input space to $\phi{(x_i)}\in\mathcal{R}^{\bar{p}}$ in a reproducing kernel Hilbert space $\mathcal{H}$, where $\bar{p}$ is an arbitrary positive integer. After kernel mapping, we obtain the mapped data points $\phi{(X)} = [\phi{(x_1)},\cdots,\phi{(x_n)}]$. The similarity between each pair of data points is defined as the inner product of mapped data in the Hilbert space, i.e., $\textbf{k}(x_i,x_j)=<\phi{(x_i)},\phi{(x_j)}> = \phi{(x_i)}^T\phi{(x_j)}$, where $\textbf{k}(\cdot,\cdot):\mathcal{R}^{p\times p} \rightarrow \mathcal{R}$ is
a reproducing kernel function. In the kernel space, \eqref{eq_obj_lsnmf_orth} is reduced to
\begin{equation}
\label{eq_obj_klsnmf_orth}
\begin{aligned}
\min_{W,G} &	\frac{1}{2} \|\phi{(X)} - \phi{(X)}WG^T\|_F^2 	+ \lambda \tr{W^T D^{\phi} G},\\
& s.t. \quad W\ge 0, G\ge 0, G^TG=I_k,
\end{aligned}
\end{equation}
where $D^{\phi}$ is extended $D$ in \eqref{eq_obj_lsnmf_orth} from instance space to kernel space defined as
\begin{equation}
\begin{aligned}
D^{\phi} = & \textbf{1}_n \textbf{1}_n^T \text{diag}\left(\phi{(X)}^T\phi{(X)}\right) \\
&	+ \text{diag}\left(\phi{(X)}^T\phi{(X)}\right) \textbf{1}_n \textbf{1}_n^T - 2 \phi{(X)}^T\phi{(X)}.
\end{aligned}
\end{equation}
We expand \eqref{eq_obj_klsnmf_orth} and replace $\phi{(X)}^T\phi{(X)}$ with $K$, the kernel matrix induced by kernel function associated with the mapping $\phi(\cdot)$, giving rise to the Kernel LS-NMF (KLS-NMF):
%
%For the first approach, we expand \eqref{eq_obj_lsnmf_orth} in the kernel space by replacing $X^TX$ with a kernel matrix $K$ induced by a proper kernel function, giving rise to the Kernel LS-NMF (KLS-NMF):
\begin{equation}
\label{eq_klsnmf}
\begin{aligned}
	\min_{W,G}& \frac{1}{2}\tr{ K - 2KWG^T + GW^TKWG^T } \\
& + \lambda \tr{W^T D_K G}, \\ 
&  s.t. \quad W\ge 0, G\ge 0, G^TG=I_k,
\end{aligned}
\end{equation}
where $D_K = D_K^T = \textbf{1}_n \textbf{1}_n^T \text{diag}(K) + \text{diag}(K) \textbf{1}_n \textbf{1}_n^T - 2 K$.

%The second approach is to learn nonlinear structures of the data on manifold; for example, by incorporating the graph Laplacian into \eqref{eq_obj_lsnmf_orth}, we get the Graph LS-NMF (GLS-NMF):
%\begin{equation}
%\label{eq_glsnmf}
%\begin{aligned}
%\min_{W\ge 0,G\ge 0} & \frac{1}{2}\|X  -  XWG^T \|_F^2 +  \lambda \tr{W^T DG}  \\
%&	+ \frac{\gamma}{2}\tr{G^T L^x G} \quad s.t. \quad G^TG = I_k,
%\end{aligned}
%\end{equation}
%where $\gamma \ge 0$ is a balancing parameter, and $L^x = D^x - W^x$ is the graph Laplacian defined in \eqref{eq_manifold}. The optimization and theoretical analysis of the proposed models are provided in the following sections.

%\begin{remark}
%This paper will focus on GLS-NMF and provide its optimization and theoretical analysis. Note that similar analysis is also applicable to KLS-NMF by setting $\gamma=0$ and replacing $X^TX$ with a kernel matrix $K$. It is also worth mentioning that our method can be extended to multiple-kernel or manifold learning scenario. Since the future extension is out of the scope of this paper, we do not further explore it here. 
%\end{remark}

\begin{remark}
In this paper, we aim at providing a new NMF method to take both local and global nonlinear relationships of the data into consideration. It is also worth mentioning that our method can be extended to multiple-kernel scenario. Since the future extension is out of the scope of this paper, we do not further explore it here. 
\end{remark}

\section{Optimization}
\label{sec_optimization}
We solve \eqref{eq_klsnmf} using an iterative update algorithm and element-wisely update $W$ and $G$ as follows:
%\begin{small}
%\begin{align}
%\label{eq_update_W}
%W_{ik} &	\leftarrow W_{ik} \sqrt{\frac{ (\Ab)_{ik} }{ (\Cb W \Db )_{ik} + \lambda (\Bb)_{ik} } }	\\ 
%\label{eq_update_G}
%G_{ik} &	\leftarrow G_{ik} \sqrt{ \frac{  (\A)_{ik} + GG^T ( \lambda \B )_{ik} }
%										{ \lambda  (\B)_{ik} + G G^T ( \A  )_{ik} } }	
%\end{align}
%\end{small}
%
\begin{align}
\label{eq_update_W}
W_{ik} &	\leftarrow W_{ik} \sqrt{\frac{ (\Ab)_{ik} }{ (\Cb W \Db )_{ik} + \lambda (\Bb)_{ik} } }	\\ 
\label{eq_update_G}
G_{ik} &	\leftarrow G_{ik} \sqrt{ \frac{  (\A)_{ik} + (\lambda GG^T \B )_{ik} }
										{ \lambda  (\B)_{ik} + ( G G^T \A  )_{ik} } }	
\end{align}
By counting dominating multiplications, it is seen that the complexity of \eqref{eq_update_W} and \eqref{eq_update_G} per iteration is $O(n^2p + n^2k)$. The correctness and convergence proofs of the updates are provided in the following section.

\section{Correctness and Convergence}
\label{sec_proof}
In this section, we will present theoretical results regarding the updates of \eqref{eq_update_W} and \eqref{eq_update_G}, respectively.

\subsection{Correctness and Convergence of \eqref{eq_update_W}}
\label{sec_proof_W}

We present two results regarding the update rule of \eqref{eq_update_W}: 1) When convergent, the limiting solution of \eqref{eq_update_W} satisfies the KKT condition. 2) The iteration of \eqref{eq_update_W} converges. The two results are established in \cref{thm_correct_W,thm_convergence_W}, respectively.

\begin{theorem}
\label{thm_correct_W}
Fixing $G$, the limiting solution of the update rule in \eqref{eq_update_W} satisfies the KKT condition.
\end{theorem}

\begin{proof} %[Proof of \cref{thm_conv_G}]
Fixing $G$, the subproblem for $W$ is
%\begin{small}
\begin{equation}
\label{eq_sub_W}
\begin{aligned}
\min_{ W\ge 0 } &	\frac{1}{2} \tr{ - 2KWG^T + GW^TKWG^T } \\
&	+ \lambda \tr{W^T D_K G},
\end{aligned}
\end{equation}
%\end{small}
%
%\begin{equation}
%\label{eq_sub_W}
%\min_{W\ge 0} \frac{1}{2}\|X - XWG^T\|_F^2 + \lambda \tr{W^T DG}.
%\end{equation}
Imposing the non-negativity constraint $W_{ik}\ge0$, we introduce the Lagrangian multipliers $\Psi = [\psi_{ij}]$ and the Lagrangian function
%\begin{small}
\begin{equation}
\begin{aligned}
\mathcal{L}_{W} =	& \frac{1}{2} \tr{ - 2KWG^T + GW^TKWG^T } \\
					& + \lambda \tr{W^T D_K G} + \tr{\Psi W^T},
\end{aligned}
\end{equation}
%\end{small}
%
%\begin{small}
%\begin{equation}
%\begin{aligned}
%\mathcal{L}_{W} =	& \frac{1}{2}\|X  -  XWG^T \|_F^2  +  \lambda \tr{W^T DG}  +  \tr{\Psi W^T}	\\
%				=	& \frac{1}{2}\tr{X^TX  -  2X^T XWG^T  +  GW^T X^TX WG^T}	\\
%					& + \lambda \tr{W^T DG} + \tr{\Psi W^T}.
%\end{aligned}
%\end{equation}
%\end{small}
The gradient of $\mathcal{L}_{W}$ gives 
\begin{equation}
\frac{\partial \mathcal{L}_{W}}{\partial W} = - KG + \lambda D_KG +  K W G^TG + \Psi.
\end{equation}
%
%\begin{equation}
%\frac{\partial \mathcal{L}_{W}}{\partial W} = - X^TXG + DG +  X^TX W G^TG - \Psi.
%\end{equation}
For ease of notation, we denote $\bar{A} = KG$, $\bar{B} = D_KG$, $\bar{C}	= K$, and $\bar{D}	= G^TG$. By the complementary slackness condition, we obtain
\begin{equation}
\label{eq_fixed_W}
(-\bar{A} + \lambda\bar{B} +  \bar{C} W \bar{D} )_{ik} W_{ik} = \psi_{ik} W_{ik} = 0.
\end{equation}
Note that \eqref{eq_fixed_W} provides the fixed point condition that the limiting solution should satisfy. It is easy to see that the limiting solution of \eqref{eq_update_W} satisfies \eqref{eq_fixed_W}, which is described as follows. At convergence, \eqref{eq_update_W} gives
\begin{equation}
\label{eq_fixed_W_2}
W_{ik} 	= W_{ik} \sqrt{\frac{ (\bar{A})_{ik} }{ (\bar{C} W \bar{D} )_{ik} + \lambda (\bar{B})_{ik} } },
\end{equation}
which is reduced to
\begin{equation}
\label{eq_comp}
(-\bar{A} + \lambda\bar{B} + \bar{C} W \bar{D} )_{ik} W_{ik}^2  = 0,
\end{equation}
by simple algebra. It is easy to see that \eqref{eq_fixed_W} and \eqref{eq_comp} are identical in that both of them enforce either $W_{ik}=0$ or $(-\bar{A} + \lambda\bar{B} + \bar{C} W \bar{D} )_{ik} = 0$.
\end{proof}

Next, we prove the convergence of the iterative update as stated in \cref{thm_convergence_W}.
\begin{theorem}
\label{thm_convergence_W}
For fixed $G$, \eqref{eq_sub_W}, as well as \eqref{eq_klsnmf}, is monotonically decreasing under the update rule in \eqref{eq_update_W}.
\end{theorem}

In this proof, we use an auxiliary function approach \cite{lee2001algorithms} with relevant definition and propositions given below.
\begin{definition}
\label{def_aux}
A function $J(H,H')$ is called an auxiliary function of $L(H)$ if for any $H$ and $H'$ the following are satisfied
\begin{equation}
J(H,H') \ge L(H),\quad J(H,H) = L(H).
\end{equation}
\end{definition}

\begin{prop}
\label{prop_aux}
Given a function $L(H)$ and its auxiliary function $J(H,H')$, if we define a variable sequence $\{H^{(t)}\}$ with
\begin{equation}
\label{eq_aux_sol}
H^{(t+1)} = \argmin_{H} J(H,H^{(t)}),
\end{equation}
then the value sequence, $\{L(H^{(t)})\}$, is decreasing due to the following chain of inequalities:
\begin{equation}
L(H^{(t)})  =  J(H^{(t)},H^{(t)})  \ge  J(H^{(t+1)},H^{(t)})  \ge  L(H^{(t+1)}). \nonumber
\end{equation}
\end{prop}

\begin{prop}[\cite{ding2010convex}]
\label{prop_inequality}
For any matrices $\Gamma\in\mathcal{R}_{+}^{n\times n}$, $\Omega\in\mathcal{R}_{+}^{k\times k}$, $S\in\mathcal{R}_{+}^{n\times k}$, and $S'\in\mathcal{R}_{+}^{n\times k}$, with $\Gamma$ and $\Omega$ being symmetric, the following inequality holds:
\begin{equation}
\label{eq_ineq_prop}
\sum_{i=1}^{n}\sum_{s=1}^{k}\frac{(\Gamma S' \Omega)_{is} S_{is}^2}{S'_{is}} \ge \tr{S^T \Gamma S \Omega}.
\end{equation}
\end{prop}

With the aid of \cref{def_aux} and \cref{prop_aux,prop_inequality}, we prove \cref{thm_convergence_W} in the following.
\begin{proof}[Proof of \cref{thm_convergence_W}]
For fixed $G$, the objective function in \eqref{eq_sub_W} can be written as
\begin{equation}
\label{eq_obj_W}
P(W) = \tr{ - W^T \bar{A}  +  \frac{1}{2} W^T \bar{C} W \bar{D}  +  \lambda W^T \bar{B} }  +  \frac{1}{2}\tr{\bar{C}}. \nonumber
\end{equation}
First, we show that the function $\bar{P}(W,W')$ defined in \eqref{eq_aux_W} is an auxiliary function of $P(W)$:
%\begin{small}
%\begin{equation}
%\label{eq_aux_W}
%\begin{aligned}
%&	\bar{P}(W,W') = \frac{1}{2}\tr{\bar{C}} -\sum_{ik} \bar{A}_{ik}W'_{ik} \Big(1+\log\frac{W_{ik}}{W'_{ik}} \Big) \\
%&	+ \frac{1}{2}\sum_{ik}\frac{(\bar{C} W' \bar{D} )_{ik}W_{ik}^2}{W'_{ik}}
%					+ \lambda  \sum_{ik} \bar{B}_{ik}\frac{W_{ik}^2+{W'}^2_{ik}}{2W'_{ik}}.
%\end{aligned}
%\end{equation}
%
\begin{equation}
\label{eq_aux_W}
\begin{aligned}
&	\bar{P}(W,W') \\
= & \frac{1}{2}\tr{\bar{C}} -\sum_{ik} \bar{A}_{ik}W'_{ik} \Big(1+\log\frac{W_{ik}}{W'_{ik}} \Big) \\
&	+ \frac{1}{2}\sum_{ik}\frac{(\bar{C} W' \bar{D} )_{ik}W_{ik}^2}{W'_{ik}}
	\\ & 	+ \lambda  \sum_{ik} \bar{B}_{ik}\frac{W_{ik}^2+{W'}^2_{ik}}{2W'_{ik}}.
\end{aligned}
\end{equation}
%\end{small}
To show this equation, we find the upper-bounds and lower-bounds for the positive and negative terms in $P(W)$, respectively. For the positive terms, we use \cref{prop_inequality} and the inequality $a \le (a^2+b^2)/2b$ for $a,b\ge 0$ to get the following upper-bounds:
%\begin{small}
\begin{equation}
\label{eq_uppers_W}
\begin{aligned}
	\tr{W^T \bar{B} }	=  \sum_{ik}\bar{B}_{ik}W_{ik} &\le  \sum_{ik} \bar{B}_{ik}\frac{W_{ik}^2+{W'}^2_{ik}}{2W'_{ik}},	\\ 
	\tr{W^T \bar{C} W \bar{D} }	&\le  \sum_{ik}\frac{(\bar{C} W' \bar{D} )_{ik}W_{ik}^2}{W'_{ik}}.	
\end{aligned}
\end{equation}
%\end{small}
For the negative term, we use the inequality $a \ge 1+ \log a$ for $a\ge 0$ to get the following lower-bound:
\begin{equation}
\label{eq_lowers_W}
\begin{aligned}
 \tr{W^T   \bar{A} }  & =  \sum_{ik}  \bar{A}_{ik}W_{ik}	\\
 &  \ge  \sum_{ik}  \bar{A}_{ik}W'_{ik}  \Big(   1+\log\frac{W_{ik}}{W'_{ik}}  \Big).
\end{aligned}
\end{equation}
Combining these bounds, we get the auxiliary function $\bar{P}(W,W')$ for $P(W)$. Next, we will show that the update of \eqref{eq_update_W} essentially follows \eqref{eq_aux_sol}, then according to \cref{prop_aux} we can conclude the proof. To show this, the remaining problem is to find the global minimum of \eqref{eq_aux_W}. For this, we first prove that \eqref{eq_aux_W} is convex.

The first-order derivative of $\bar{P}(W,W')$ is
%\begin{equation}
%\label{eq_aux_W_derivative}
%\begin{aligned}
%	&	\frac{\partial \bar{P}(W,W')}{\partial W_{ik}}	\\
%=	& -\frac{\bar{A}_{ik}W'_{ik}}{W_{ik}} + \frac{(\bar{C} W' \bar{D} )_{ik} W_{ik} }{W'_{ik}} 
%					+ \lambda  \frac{\bar{B}_{ik} W_{ik} }{W'_{ik}}.
%\end{aligned}
%\end{equation}
%
\begin{equation}
\label{eq_aux_W_derivative}
\begin{aligned}
\!	\frac{\partial \bar{P}(W,W')}{\partial W_{ik}}	= -\frac{\bar{A}_{ik}W'_{ik}}{W_{ik}} + \frac{(\bar{C} W' \bar{D} )_{ik} W_{ik} }{W'_{ik}} 
					+ \lambda  \frac{\bar{B}_{ik} W_{ik} }{W'_{ik}}.
\end{aligned}
\end{equation}
Then the Hessian of $H(W,W')$ can be obtained element-wisely as 
%\begin{equation}
%\label{eq_aux_W_hessian}
%\begin{aligned}
%	&	\frac{\partial^2 \bar{P}(W,W')}{\partial W_{ik} \partial W_{jl}} \\
%=	& \delta_{ij}\delta_{jk} \left( \frac{\bar{A}_{ik}W'_{ik}}{W_{ik}^2} 
%						 +  \frac{(\bar{C} W' \bar{D} )_{ik} +\lambda \bar{B}_{ik} }{W'_{ik}} \right),
%\end{aligned}
%\end{equation}
%
\begin{equation}
\label{eq_aux_W_hessian}
\begin{aligned}
	\frac{\partial^2 \bar{P}(W,W')}{\partial W_{ik} \partial W_{jl}} 
	= \delta_{ij}\delta_{jk} \left( \frac{\bar{A}_{ik}W'_{ik}}{W_{ik}^2} 
						 +  \frac{(\bar{C} W' \bar{D} )_{ik} +\lambda \bar{B}_{ik} }{W'_{ik}} \right),
\end{aligned}
\end{equation}
where $\delta_{ij}$ is delta function that returns 1 if $i=j$ or 0 otherwise. It is seen that the Hessian matrix of $\bar{P}(W,W')$ has zero elements off diagonal and nonzero elements on diagonal, and thus is positive definite. Therefore, $\bar{P}(W,W')$ is convex and achieves the global optimum by its first-order optimality condition, i.e., \eqref{eq_aux_W_derivative} = 0, which gives rise to
\begin{equation}
\label{eq_aux_W_derivative2}
\begin{aligned}
\frac{\bar{A}_{ik}W'_{ik}}{W_{ik}} = \frac{(\bar{C} W' \bar{D} )_{ik} W_{ik} }{W'_{ik}} + \lambda \frac{\bar{B}_{ik} W_{ik} }{W'_{ik}}.
\end{aligned}
\end{equation}
\eqref{eq_aux_W_derivative2} can be further reduced to
\begin{equation}
\label{eq_update_proof2}
\begin{aligned}
W_{ik} = W'_{ik} \sqrt{\frac{ \bar{A}_{ik} }{ (\bar{C} W' \bar{D} )_{ik} + \lambda \bar{B}_{ik} } }.
\end{aligned}
\end{equation}
Define $W^{(t+1)}=W$, and $W^{(t)} = W'$, we can see that \eqref{eq_sub_W} is decreasing under the update of \eqref{eq_update_proof2}. Substituting $\bar{A}$, $\bar{B}$, $\bar{C}$, $\bar{D}$, we recover \eqref{eq_update_W}.
\end{proof}

\subsection{Correctness and Convergence of \eqref{eq_update_G}}
Fixing $W$, we need to solve the following optimization problem for $G$:
\begin{equation}
\label{eq_sub_G}
\begin{aligned}
\argmin_{G} & =  \frac{1}{2}\tr{ - 2KWG^T + GW^TKWG^T } \\
&	 +  \lambda \tr{W^T D_K G},   s.t.\quad G\ge 0, G^TG = \Lambda,
\end{aligned}
\end{equation}
%
%\begin{equation}
%\begin{aligned}
%\argmin_{G} & = \frac{1}{2}	\|X  -  XWG^T \|_F^2  +  \lambda \tr{W^T DG} \\
%& +  \frac{\gamma}{2}\tr{G^T L^x G} \quad	s.t.\quad G\ge 0, G^TG = \Lambda,
%\end{aligned}
%\end{equation}
where $\Lambda$ is nonnegative and diagonal. We introduce the Lagrangian multipliers $\Theta$, which is symmetric and has size $k\times k$. Then the Lagrangian function to be minimized gives rise to
%\begin{small}
\begin{equation}
\begin{aligned}
\mathcal{L}_{G} =	&	\frac{1}{2}\tr{ - 2KWG^T + GW^TKWG^T }  	\\
					&	+ \lambda \tr{W^T D_K G} + \frac{1}{2}\tr{\Theta (G^TG - \Lambda)}\\
				=	&	\frac{1}{2}\textbf{Tr}( - 2KWG^T + GW^TKWG^T 	\\
					&	\qquad + 2\lambda W^T D_K G +  G\Theta G^T  ) - \xi	\\
%	\nonumber\end{aligned}\end{equation}\begin{equation}\begin{aligned}
				=	& 	 \frac{1}{2}   \textbf{Tr}  (  -  2AG^T  +  GCG^T  +  2\lambda B G^T  +  G\Theta G^T)  -  \xi	\\
				=	& 	\frac{1}{2}\textbf{Tr} ( - 2 A G^T + 2\lambda B G^T  \\
					&	\qquad + G(C + \Theta)^{+} G^T - G(C + \Theta)^{-} G^T ) - \xi,
\end{aligned}
\end{equation}
%\end{small}
%
%\begin{small}
%\begin{equation}
%\begin{aligned}
%&	\mathcal{L}_{G} =	\frac{1}{2}\|X - XWG^T \|_F^2 + \lambda \tr{W^T DG} 	\\ %\end{aligned}\end{equation}\begin{equation}\begin{aligned}\nonumber
%					&	+ \frac{\gamma}{2} \tr{G^T L^x G} + \frac{1}{2}\tr{\Theta (G^TG - \Lambda)}\\
%				=	& 	\frac{1}{2}\textbf{Tr} (X^TX  -  2X^TXWG^T  +  GW^TX^TXWG^T \\	
%					&	+ 2\lambda DW G^T  +  G\Theta G^T)  +  \frac{\gamma}{2}\tr{G^TL^x G}  -  \xi	\\
%				=	& 	\frac{1}{2}\textbf{Tr} (X^TX - 2AG^T + GCG^T + 2\lambda B G^T \\
%					&	+ G\Theta G^T) + \frac{\gamma}{2} \tr{G^T D^xG - G^T W^xG} -  \xi	\\
%				=	& 	\frac{1}{2}\textbf{Tr} (X^TX  -  2 A G^T  +  2\lambda   B G^T  +  G(C  +  \Theta)^{+} G^T  \\
%					&	 -  G(C  +  \Theta)^{-} G^T )  +  \frac{\gamma}{2}   \tr{G^T D^xG  -  G^T W^xG} - \xi,
%\end{aligned}
%\end{equation}
%\end{small}
where we define $\xi = \frac{1}{2}\tr{\Theta \Lambda}$, $A = \A$, $B = \B$, and $C = \C$ for easier notation, and $M^+$, $M^-$ to be two nonnegative matrices for a nonnegative matrix $M$ such that $(M^+ - M^-) = M$. The gradient of $\mathcal{L}_G$ is
\begin{equation}
\label{eq_lag_G_gradient}
\frac{\partial \mathcal{L}_{G}}{\partial G} = -2 A + 2GC + 2\lambda B + 2G\Theta.
\end{equation}
%\begin{equation}
%\label{eq_lag_G_gradient}
%\frac{\partial \mathcal{L}_{G}}{\partial G} = -2 A + 2GC + 2\lambda B + 2G\Theta + 2\gamma L^xG.
%\end{equation}
Then the KKT complementarity condition gives
\begin{equation}
\label{eq_fix_cond_1}
(-A + GC + \lambda B + G\Theta )_{ik} G_{ik} = 0,
\end{equation}
%\begin{equation}
%\label{eq_fix_cond_1}
%(-A + GC + \lambda B + G\Theta + \gamma L^xG)_{ik} G_{ik} = 0,
%\end{equation}
%which is a fixed point relation that the local minimum for $G$ must hold. Following \cref{sec_proof_W}, noting that $C+\Theta = (C+\Theta)^+ + (C+\Theta)^-$ and $L^x = D^x - W^x$ we give an update as follows:
%\begin{equation}
%\label{eq_update_proof}
%G_{ik} \leftarrow G_{ik} \sqrt{ \frac{A_{ik} + (G (C+\Theta)^- )_{ik} + \gamma (W^x G)_{ik} }{\lambda B_{ik} + (G (C+\Theta)^+ )_{ik} + \gamma (D^xG)_{ik} }}.
%\end{equation}
%
%which is a fixed point relation that the local minimum for $G$ must hold. Following \cref{sec_proof_W}, noting that $C+\Theta = (C+\Theta)^+ + (C+\Theta)^-$ we give an update as follows:
which is a fixed point relation that the local minimum for $G$ must hold. Following the previous subsection, noting that $$C+\Theta = (C+\Theta)^+ - (C+\Theta)^-$$ we give an update as follows:
\begin{equation}
\label{eq_update_proof}
G_{ik} \leftarrow G_{ik} \sqrt{ \frac{A_{ik} + (G (C+\Theta)^- )_{ik} }{\lambda B_{ik} + (G (C+\Theta)^+ )_{ik} }}.
\end{equation}
To show that the update of \eqref{eq_update_proof} will converge to a local minimum, we will show two results: the convergence of the update algorithm and the correctness of the converged solution. 

From \eqref{eq_update_proof}, it is easy to show that, at convergence, the solution satisfies the following condition:
%\begin{equation}
%\label{eq_fix_cond_2}
%(-A + GC + \lambda B + G\Theta + \gamma L^x G)_{ik} G_{ik}^2 = 0,
%\end{equation}
%
\begin{equation}
\label{eq_fix_cond_2}
(-A + GC + \lambda B + G\Theta)_{ik} G_{ik}^2 = 0,
\end{equation}
which is the fixed point condition in \eqref{eq_fix_cond_1}. Hence, the correctness of the converged solution can be verified.

The convergence is assured by the following theorem.
\begin{theorem}
\label{thm_conv_G}
For fixed $W$, the Lagrangian function $\mathcal{L}_{G}$ is monotonically decreasing under the update rule in \eqref{eq_update_proof}.
\end{theorem}

\begin{proof}
To prove \cref{thm_conv_G}, we use the auxiliary function approach. For ease of notation, we define $E = C+\Theta$.

First, we find upper-bounds for each positive term in $\mathcal{L}_{G}$. By inequality $a \le (a^2+b^2)/2b$ for $a,b\ge 0$, we get 
\begin{equation}
\label{eq_uppers_G_1}
\begin{aligned}
\tr{G^TB }	 = \sum_{ik} B_{ik} G^{ik} \le  \sum_{ik} B_{ik}\frac{G_{ik}^2+{G'}^2_{ik}}{2G'_{ik}}.
\end{aligned}
\end{equation}
Then, according to \cref{prop_inequality}, by setting $\Gamma$ or $S$ to be identity matrices, we get the following two upper-bounds
\begin{equation}
\label{eq_uppers_G_2}
\begin{aligned}
\tr{G E^+ G^T}	\le & \sum_{ik}\frac{(G'E^+ )_{ik}G_{ik}^2}{G'_{ik}}	\\
%\tr{G^T D^x G} \le & \frac{(D^x G')_{ik} G_{ik}^2}{G'_{ik}}.
\end{aligned}
\end{equation}
Then, by the inequalities $a \ge 1+ \log a$ for $a\ge 0$, we get the following lower-bounds for negative terms:
%\begin{small}
\begin{equation}
\label{eq_lowers}
\begin{aligned}
\tr{G^TA}	\ge & \sum_{ik} A_{ik}G'_{ik} \Big(1+\log\frac{G_{ik}}{G'_{ik}} \Big)	\\
\tr{G E^- G^T} \ge & \sum_{ikl}E^-_{kl} G'_{ik}G'_{il} \Big(1+\log \frac{G_{ik}G_{il}}{G'_{ik}G'_{il}} \Big).	\\
%\tr{G^T W^x G} \ge & \sum_{ikl}W^x_{kl}G'_{ki}G'_{li}\Big(1+\log\frac{G_{ki}G_{li}}{G'_{ki}G'_{li}} \Big).
\end{aligned}
\end{equation}
%\end{small}
Hence, combining the above bounds, we construct an auxiliary function for $\mathcal{L}_G$:
%\begin{small}
%\begin{equation}
%\label{eq_aux_G}
%\begin{aligned}
%& 	J(G,G') =  -\sum_{ik} A_{ik}G'_{ik} \Big(1+\log\frac{G_{ik}}{G'_{ik}} \Big) \\  
%&	+ \lambda \sum_{ik} B_{ik}\frac{G_{ik}^2+{G'}^2_{ik}}{2G'_{ik}}	+ \frac{1}{2}\sum_{ik}\frac{(G'E^+ )_{ik}G_{ik}^2}{G'_{ik}}	\\
%&	- \frac{1}{2}\sum_{ikl}E^-_{kl} G'_{ik}G'_{il} \Big(1+\log \frac{G_{ik}G_{il}}{G'_{ik}G'_{il}} \Big) + \frac{1}{2}\tr{K}. 
%\end{aligned}
%\end{equation}
%
%\begin{equation}
%\label{eq_aux_G}
%\begin{aligned}
%&	J(G,G') \\
%= &  -\sum_{ik} A_{ik}G'_{ik} \Big(1+\log\frac{G_{ik}}{G'_{ik}} \Big) \\  
%&	+ \lambda \sum_{ik} B_{ik}\frac{G_{ik}^2+{G'}^2_{ik}}{2G'_{ik}} \\
%&	+ \frac{1}{2}\sum_{ik}\frac{(G'E^+ )_{ik}G_{ik}^2}{G'_{ik}}	\\
%&	- \frac{1}{2}\sum_{ikl}E^-_{kl} G'_{ik}G'_{il} \Big(1+\log \frac{G_{ik}G_{il}}{G'_{ik}G'_{il}} \Big) + \frac{1}{2}\tr{K}. 
%\end{aligned}
%\end{equation}
%\end{small}
%
%\begin{small}
\begin{equation}
\label{eq_aux_G}
\begin{aligned}
& 	J(G,G') =  -\sum_{ik} A_{ik}G'_{ik} \Big(1+\log\frac{G_{ik}}{G'_{ik}} \Big) \\  
&	+ \lambda \sum_{ik} B_{ik}\frac{G_{ik}^2+{G'}^2_{ik}}{2G'_{ik}}	+ \frac{1}{2}\sum_{ik}\frac{(G'E^+ )_{ik}G_{ik}^2}{G'_{ik}}	\\
\end{aligned}\end{equation}\begin{equation}\begin{aligned}\nonumber
&	- \frac{1}{2}\sum_{ikl}E^-_{kl} G'_{ik}G'_{il} \Big(1+\log \frac{G_{ik}G_{il}}{G'_{ik}G'_{il}} \Big)	\\
&	- \frac{\gamma}{2} \sum_{ikl}(W^x)_{kl}G'_{ki}G'_{li}\Big(1+\log\frac{G_{ki}G_{li}}{G'_{ki}G'_{li}}\Big) \\
&	+ \frac{\gamma}{2} \frac{(D^x G')_{ik}}{G'_{ik}} {G_{ik}^2} + \frac{1}{2}\tr{X^TX}. 
\end{aligned}
\end{equation}
We take the first order derivative of \eqref{eq_aux_G}, then we get
%\begin{small}
%\begin{equation}
%\label{eq_aux_G_derivative}
%\begin{aligned}
%&\frac{\partial J(G,G')}{\partial G_{ik}} \\
%= & - \frac{ A_{ik}G'_{ik}}{G_{ik}}  +   \lambda \frac{B_{ik}}{G'_{ik}}G_{ik}\\
%&  +  \frac{(G'E^+ )_{ik}}{G'_{ik}} G_{ik} -  \frac{(G' E^- )_{ik}G'_{ik} }{G_{ik}}.
%\end{aligned}
%\end{equation}
%\end{small}
%
%\begin{small}
\begin{equation}
\label{eq_aux_G_derivative}
\begin{aligned}
& \frac{\partial J(G,G')}{\partial G_{ik}}  =  - \frac{ A_{ik}G'_{ik}}{G_{ik}}  +   \lambda \frac{B_{ik}}{G'_{ik}}G_{ik} +  \frac{(G'E^+ )_{ik}}{G'_{ik}} G_{ik}\\
& -  \frac{(G' E^- )_{ik}G'_{ik} }{G_{ik}}  +  \gamma \frac{(D^x G')_{ik}}{G'_{ik}} {G_{ik}}  -  \gamma \frac{(W^x G')_{ik} G'_{ik}}{G_{ik}}.
\end{aligned}
\end{equation}
%\end{small}
Further, we can get the Hessian of \eqref{eq_aux_G} by taking the second order derivative:
%\begin{small}
%\begin{equation}
%\begin{aligned}
%&	\frac{\partial^2 Z(G,G')}{\partial G_{ik} \partial G_{jl}} \\
%= & \delta_{ij}\delta_{kl} \Bigg( \frac{ A_{ik}G'_{ik}}{G_{ik}^2} 
%	+ \lambda \frac{B_{ik}}{G'_{ik}}	\\
%& \qquad\qquad\qquad	+ \frac{(G'E^+ )_{ik}}{G'_{ik}}	+ \frac{(G' E^- )_{ik}G'_{ik} }{G_{ik}^2} \Bigg).
%\end{aligned}
%\end{equation}
%\end{small}
%
%\begin{small}
\begin{equation}
\begin{aligned}
&	\frac{\partial^2 Z(G,G')}{\partial G_{ik} \partial G_{jl}}	= \delta_{ij}\delta_{kl} \Big( \frac{ A_{ik}G'_{ik}}{G_{ik}^2} 
			+ \lambda \frac{B_{ik}}{G'_{ik}}+ \frac{(G'E^+ )_{ik}}{G'_{ik}}	\\
&	  + \frac{(G' E^- )_{ik}G'_{ik} }{G_{ik}^2} +\gamma \frac{(D^x G')_{ik}}{G'_{ik}} 	+ \gamma \frac{(W^x G')_{ik} G'_{ik}}{G_{ik}^2} \Big).
%	=	\delta_{ij}\delta_{kl} \Big( 2\frac{\Delta_1 G'_{ik}}{G_{ik}^2} + 2\frac{\Delta_2}{G'_{ik}} \Big),
\end{aligned}
\end{equation}
%\end{small}
It is easy to verify that the Hessian matrix has zero elements off diagonal, and nonnegative values on diagonal. Therefore, $J(G,G')$ is convex in $G$ and its global minimum is obtained by its first order optimality condition, \eqref{eq_aux_G_derivative} = 0, which gives rise to
%\begin{small}
\begin{equation}
%\label{eq_update_proof}
G_{ik} = G'_{ik} \sqrt{ \frac{A_{ik} + (G' E^- )_{ik} }{\lambda B_{ik} + (G'E^+ )_{ik} }}.
\end{equation}
%\end{small}
%
%\begin{small}
%\begin{equation}
%%\label{eq_update_proof}
%G_{ik} = G'_{ik} \sqrt{ \frac{A_{ik} + (G' E^- )_{ik} + \gamma (W^x G')_{ik} }{\lambda B_{ik} + (G'E^+ )_{ik} + \gamma (D^xG')_{ik} }}.
%\end{equation}
%\end{small}

According to \cref{prop_aux}, by setting $G^{(t+1)} = G$ and $G^{(t)} = G'$, we recover \eqref{eq_update_proof} and it is easy to see that $\mathcal{L}_G(G)$ is decreasing under \eqref{eq_update_proof}.
\end{proof}

It is seen that in \eqref{eq_update_proof}, the multipliers $\Theta$ is yet to be determined. By the first order optimality condition of $\mathcal{L}_G$, i.e., \eqref{eq_lag_G_gradient} = 0, we can see that
%\begin{small}
\begin{equation}
\begin{aligned}
	&	G^T(-A + GC + \lambda B + G\Theta )	\\
=	&	- G^T A + G^TG C  +  \lambda G^T B + G^TG \Theta 	\\
=	&	- G^T A + C + \lambda G^T B + \Theta \\
=	& 	\quad 0,
\end{aligned}
\end{equation}
%\end{small}
hence
%\begin{small}
\begin{equation}
\begin{aligned}
E = G^T A - \lambda G^T B.
\end{aligned}
\end{equation}
%\end{small}
Note that by defining $E^- = \lambda G^T B$, and $E^+ = G^T A$, we have $E^+ - E^- = E$ and $E^+ \ge0$, $E^-\ge 0$. Substituting $E^+$ and $E^-$ into \eqref{eq_update_proof}, we get the update rule in \eqref{eq_update_G}.
%
%Note that by defining $E^- = \lambda G^T B + \gamma G^T D^xG$, and $E^+ = G^T A +\gamma G^T W^x G$, we have $E^+ - E^- = E$ and $E^+ \ge0$, $E^-\ge 0$. Substituting $E^+$ and $E^-$ into \eqref{eq_update_proof}, we get the update rule in \eqref{eq_update_G}.

\begin{remark}
So far, a conclusion 
%similar to \cite{ding2006orthogonal}
can be drawn that by alternatively updating $W$ and $G$, the objective function in \eqref{eq_klsnmf} will decrease and the value sequence converges. We set $\varUpsilon = [W^T,G^T]^T\in\mathcal{R}^{2n \times k}$, and regard the updates of \eqref{eq_update_W} and \eqref{eq_update_G} as a mapping $\varUpsilon^{(t+1)} = \mathcal{M}(\varUpsilon^{(t)})$, then at convergence we have $\varUpsilon^* = \mathcal{M}(\varUpsilon^*)$. Following \cite{ding2010convex,xu1996convergence}, with non-negativity constraint enforced, we expand $\varUpsilon \approxeq \mathcal{M}(\varUpsilon^*) + (\partial \mathcal{M} / \partial \varUpsilon)(\varUpsilon - \varUpsilon^*)$, which indicates that $\|\varUpsilon^{(t+1)} -\varUpsilon^*\|\le \|\partial \mathcal{M} / \partial \varUpsilon\| \cdot \|\varUpsilon^{(t)} -\varUpsilon^*\|$ under an appropriate matrix norm. In general, $\|\partial \mathcal{M} / \partial \varUpsilon\|\not= 0$, hence the updates of \eqref{eq_update_W} and \eqref{eq_update_G} roughly have a first-order convergence rate.
\end{remark}

\section{Experiments}
\label{sec_experiments}
In this section, we conduct experiments to verify the effectiveness of the proposed KLS-NMF. We will present the evaluation metrics, benchmark datasets, algorithms in comparison, and experimental results in detail.

\begin{table}[!tb] %[Semeion]
	%\small
	\centering
	\caption{Clustering Performance on Semeion }
	\resizebox{0.48\textwidth}{!}{
		%\begin{tabular}{ |c|c|c|c|c|c||c|}
		\begin{tabular}{|p{1cm}<{\centering} | p{2cm}<{\centering} | p{2cm}<{\centering} | p{2cm}<{\centering} | p{2cm}<{\centering} | p{2cm}<{\centering} || p{2cm}<{\centering} |}
			\hline
			\multirow{2}{1cm}{\centering N}& \multicolumn{6}{c|}{Accuracy (\%)}\\	
			\cline{2-7}
			\multirow{2}{1cm}{}
			& WNMF				& RMNMF				& CNMF	 			& KNMF				& ONMF 				& KLS-NMF						\\ \hline
			2	& 87.57$\pm$10.53	& 87.58$\pm$10.64	& 88.18$\pm$10.02	& 87.88$\pm$10.73	& 87.10$\pm$11.63	& \bf{88.86$\pm$10.54}	\\
			3	& 80.31$\pm$09.91	& 78.23$\pm$09.17	& 80.23$\pm$10.51	& 80.58$\pm$10.52	& 79.43$\pm$07.39	& \bf{82.88$\pm$08.53}			\\
			4	& 71.95$\pm$06.07	& 65.22$\pm$07.80	& 70.32$\pm$08.91	& 67.88$\pm$10.86	& 70.80$\pm$08.62	& \bf{75.32$\pm$11.16}			\\
			5	& 70.24$\pm$06.77	& 62.33$\pm$07.31	& 67.61$\pm$10.23	& 64.40$\pm$07.41	& 64.36$\pm$08.39	& \bf{75.26$\pm$07.33}			\\
			6	& 58.25$\pm$05.69	& 54.67$\pm$06.88	& 57.50$\pm$06.14	& 61.71$\pm$09.32	& 61.57$\pm$06.77	& \bf{64.91$\pm$08.69}			\\
			7	& 59.32$\pm$07.24	& 52.94$\pm$06.03	& 54.42$\pm$05.89	& 61.36$\pm$05.91	& 57.68$\pm$07.48	& \bf{64.66$\pm$05.42}			\\
			8	& 59.63$\pm$07.53	& 48.23$\pm$04.31	& 53.52$\pm$04.81	& 60.33$\pm$05.64	& 58.02$\pm$06.95	& \bf{67.15$\pm$06.74}			\\
			9	& 56.35$\pm$04.12	& 44.90$\pm$02.77	& 50.16$\pm$05.59	& 56.06$\pm$05.52	& 56.63$\pm$08.88	& \bf{59.25$\pm$02.74}			\\
			10	& 55.56				& 43.57				& 45.20				& 52.54				& 49.15				& \bf{60.58}					\\	\hline
			Average	
			& 66.57				& 59.74				& 63.01				& 65.86				& 64.97				& \bf{70.99}					\\ \hline
			%---------------------
			\hline
			\multirow{2}{1cm}{\centering N}& \multicolumn{6}{c|}{NMI (\%)}\\	
			\cline{2-7}
			\multirow{2}{1cm}{}
			& WNMF				& RMNMF				& CNMF				& KNMF				& ONMF				& KLS-NMF				\\ \hline
			2	& 56.16$\pm$28.34	& 55.48$\pm$28.88	& 56.20$\pm$28.69	& 57.41$\pm$28.32	& 55.51$\pm$30.48	& \bf{60.70$\pm$30.26}	\\
			3	& 54.01$\pm$13.67	& 50.39$\pm$12.33	& 53.90$\pm$14.75	& 55.95$\pm$12.58	& 50.22$\pm$11.18	& \bf{58.68$\pm$11.41}	\\
			4	& 50.68$\pm$04.88	& 44.83$\pm$06.88	& 49.93$\pm$06.20	& 50.52$\pm$07.34	& 49.02$\pm$05.37	& \bf{58.22$\pm$09.09}	\\
			5	& 52.28$\pm$06.09	& 43.45$\pm$07.15	& 51.08$\pm$08.22	& 54.32$\pm$03.32	& 49.88$\pm$07.96	& \bf{61.15$\pm$07.27}	\\
			6	& 45.58$\pm$04.75	& 39.81$\pm$06.31	& 45.25$\pm$06.11	& 51.11$\pm$05.01	& 47.46$\pm$05.93	& \bf{55.26$\pm$07.79}	\\
			7	& 46.55$\pm$06.27	& 41.71$\pm$04.53	& 44.05$\pm$04.81	& 51.57$\pm$04.88	& 46.56$\pm$06.12	& \bf{54.07$\pm$04.08}	\\
			8	& 48.18$\pm$04.90	& 39.51$\pm$03.19	& 44.36$\pm$03.54	& 52.49$\pm$02.81	& 46.70$\pm$04.29	& \bf{58.96$\pm$04.45}	\\
			9	& 47.18$\pm$03.78	& 36.52$\pm$02.66	& 42.75$\pm$04.51	& 49.29$\pm$03.99	& 45.75$\pm$04.75	& \bf{54.43$\pm$02.45}	\\
			10	& 44.82				& 35.44				& 37.96				& 47.38				& 43.12				& \bf{54.98}			\\	\hline
			Average	
			& 49.49				& 43.02				& 47.28				& 52.23				& 48.25				& \bf{57.38}			\\	\hline
			%---------------------
			\hline
			\multirow{2}{1cm}{\centering N}& \multicolumn{6}{c|}{Purity (\%)}\\	
			\cline{2-7}
			\multirow{2}{1cm}{}
			& WNMF				& RMNMF				& CNMF				& KNMF				& ONMF				& KLS-NMF						\\ \hline
			2	& 87.57$\pm$10.53	& 87.58$\pm$10.64	& 88.18$\pm$10.02	& 87.88$\pm$10.73	& 87.10$\pm$11.63	& \bf{88.86$\pm$10.54}	\\
			3	& 80.31$\pm$09.91	& 78.23$\pm$09.17	& 80.39$\pm$10.19	& 80.67$\pm$10.35	& 79.43$\pm$07.39	& \bf{82.88$\pm$08.53}			\\
			4	& 72.33$\pm$05.76	& 67.08$\pm$06.60	& 71.91$\pm$06.45	& 71.09$\pm$07.60	& 72.06$\pm$05.92	& \bf{76.51$\pm$08.74}			\\
			5	& 70.51$\pm$06.74	& 63.77$\pm$05.81	& 69.13$\pm$07.59	& 69.25$\pm$04.18	& 67.59$\pm$06.43	& \bf{76.10$\pm$06.40}			\\
			6	& 60.91$\pm$04.53	& 56.44$\pm$05.79	& 61.03$\pm$05.25	& 65.64$\pm$06.08	& 63.45$\pm$06.24	& \bf{67.83$\pm$07.45}			\\
			7	& 60.88$\pm$06.43	& 54.69$\pm$05.57	& 57.35$\pm$05.68	& 65.02$\pm$04.32	& 61.12$\pm$06.32	& \bf{67.11$\pm$03.74}			\\
			8	& 60.58$\pm$06.55	& 49.88$\pm$03.76	& 55.72$\pm$03.92	& 63.94$\pm$03.71	& 60.13$\pm$05.82	& \bf{68.84$\pm$04.47}			\\
			9	& 59.04$\pm$04.61	& 46.18$\pm$02.82	& 52.57$\pm$05.44	& 60.18$\pm$05.00	& 59.20$\pm$06.61	& \bf{64.10$\pm$02.78}			\\
			10	& 56.56				& 45.95				& 45.20				& 52.54				& 54.74				& \bf{61.83}					\\	\hline
			Average	
			& 67.63				& 61.09				& 64.61				& 68.47				& 67.20				& \bf{72.67}					\\	\hline	
			%--------------------& & & & & & & 
		\end{tabular}
	}
	\label{tab_per_semeion}
	%\end{tiny}
\end{table}

\begin{table}[!tb] %[Jaffe]
	%\Huge
	\centering
	\caption{Clustering Performance on JAFFE }
	\resizebox{0.48\textwidth}{!}{
		%\begin{tabular}{ |c|c|c|c|c|c||c|}
		\begin{tabular}{|p{1cm}<{\centering} | p{2cm}<{\centering} | p{2cm}<{\centering} | p{2cm}<{\centering} | p{2cm}<{\centering} | p{2cm}<{\centering} || p{2cm}<{\centering} |}
			\hline
			\multirow{2}{1cm}{\centering N}& \multicolumn{6}{c|}{Accuracy (\%)}\\	
			\cline{2-7}
			\multirow{2}{1cm}{}
			& WNMF				& RMNMF					& CNMF				& KNMF 				& ONMF				& KLS-NMF						\\ \hline
			2	& 99.75$\pm$00.79	& \bf{100.0$\pm$00.00}	& 99.75$\pm$00.00	& 99.75$\pm$00.79	& 99.25$\pm$02.37	& \bf{100.0$\pm$00.00}			\\
			3	& 96.54$\pm$05.05	& 97.62$\pm$01.86		& 87.98$\pm$13.94	& 96.36$\pm$03.91	& 84.06$\pm$16.95	& \bf{98.72$\pm$01.47}	\\
			4	& 95.92$\pm$05.96	& 98.83$\pm$01.73		& 80.37$\pm$17.35	& 89.54$\pm$13.01	& 91.88$\pm$14.41	& \bf{99.07$\pm$02.04}	\\
			5	& 95.75$\pm$03.92	& 97.46$\pm$03.09		& 88.29$\pm$08.25	& 87.26$\pm$10.56	& 72.47$\pm$06.66	& \bf{98.39$\pm$02.23}	\\
			6	& 89.47$\pm$04.41	& 95.14$\pm$04.07		& 76.26$\pm$13.45	& 83.50$\pm$08.14	& 88.98$\pm$12.69	& \bf{97.80$\pm$01.14}	\\
			7	& 89.68$\pm$10.77	& 90.24$\pm$06.90		& 72.05$\pm$11.21	& 83.14$\pm$09.33	& 79.65$\pm$08.69	& \bf{96.79$\pm$02.35}	\\
			8	& 92.05$\pm$05.57	& 91.63$\pm$05.58		& 69.44$\pm$10.06	& 79.24$\pm$07.30	& 74.74$\pm$07.43	& \bf{96.52$\pm$01.61}	\\
			9	& 86.84$\pm$04.69	& 90.73$\pm$07.06		& 63.82$\pm$05.77	& 79.76$\pm$06.36	& 79.01$\pm$06.05	& \bf{95.51$\pm$01.23}	\\
			10	& 90.61				& 95.77					& 69.95				& 81.69				& 82.63				& \bf{96.24}				\\	\hline
			Average	
			& 92.96				& 95.27					& 78.66				& 86.69				& 83.63				& \bf{97.67}				\\	\hline
			%---------------------
			\hline
			\multirow{2}{1cm}{\centering N}& \multicolumn{6}{c|}{NMI (\%)}\\	
			\cline{2-7}
			\multirow{2}{1cm}{}
			& WNMF				& RMNMF					& CNMF				& KNMF				& ONMF				& KLS-NMF						\\ \hline
			2	& 98.55$\pm$04.59	& \bf{100.0$\pm$00.00}	& 98.55$\pm$04.59	& 98.55$\pm$04.59	& 96.79$\pm$10.16	& \bf{100.0$\pm$00.00}			\\
			3	& 91.29$\pm$10.58	& 92.02$\pm$05.91		& 78.83$\pm$18.03	& 89.92$\pm$10.13	& 78.25$\pm$17.32	& \bf{95.84$\pm$04.63}	\\
			4	& 91.48$\pm$10.86	& 96.98$\pm$03.88		& 75.52$\pm$17.36	& 86.39$\pm$15.37	& 92.30$\pm$09.66	& \bf{97.82$\pm$04.70}	\\
			5	& 92.94$\pm$05.56	& 95.01$\pm$05.29		& 84.42$\pm$08.55	& 85.72$\pm$08.72	& 73.86$\pm$05.60	& \bf{96.69$\pm$04.49}	\\
			6	& 85.58$\pm$05.96	& 91.76$\pm$05.45		& 73.17$\pm$13.80	& 83.17$\pm$06.95	& 88.91$\pm$10.75	& \bf{95.68$\pm$02.05}	\\
			7	& 88.18$\pm$09.17	& 87.12$\pm$05.60		& 69.79$\pm$11.35	& 85.46$\pm$05.05	& 81.43$\pm$08.65	& \bf{94.79$\pm$03.58}	\\
			8	& 91.22$\pm$04.86	& 89.09$\pm$05.20		& 66.10$\pm$11.38	& 82.18$\pm$04.27	& 81.33$\pm$06.17	& \bf{94.50$\pm$02.53}	\\
			9	& 87.20$\pm$03.18	& 89.34$\pm$05.09		& 62.37$\pm$05.03	& 83.03$\pm$04.25	& 82.49$\pm$04.57	& \bf{93.73$\pm$01.57}	\\
			10	& 89.44				& 93.54					& 70.65				& 82.38				& 84.46				& \bf{94.40}				\\	\hline
			Average		& 90.65				& 92.76					& 75.49				& 86.31				& 84.42				& \bf{95.94}		\\	\hline
			%---------------------	
			\hline
			\multirow{2}{1cm}{\centering N}& \multicolumn{6}{c|}{Purity (\%)}\\	
			\cline{2-7}
			\multirow{2}{1cm}{}
			& WNMF				& RMNMF					& CNMF				& KNMF				& ONMF				& KLS-NMF						\\	\hline
			2	& 99.75$\pm$00.79	& \bf{100.0$\pm$00.00}	& 99.75$\pm$00.79	& 99.75$\pm$00.79	& 99.25$\pm$02.37	& \bf{100.0$\pm$00.00}			\\
			3	& 96.54$\pm$05.05	& 97.62$\pm$01.86		& 88.94$\pm$11.81	& 96.36$\pm$03.91	& 86.25$\pm$13.26	& \bf{98.72$\pm$01.47}	\\
			4	& 95.92$\pm$05.96	& 98.83$\pm$01.73		& 83.29$\pm$13.59	& 90.61$\pm$11.26	& 94.11$\pm$09.71	& \bf{99.09$\pm$02.04}	\\
			5	& 95.75$\pm$03.92	& 97.46$\pm$03.09		& 88.66$\pm$07.62	& 88.28$\pm$09.17	& 76.60$\pm$06.07	& \bf{98.39$\pm$02.23}	\\
			6	& 89.47$\pm$04.41	& 95.14$\pm$04.07		& 78.30$\pm$11.83	& 84.83$\pm$06.79	& 90.41$\pm$10.32	& \bf{97.80$\pm$01.14}	\\
			7	& 90.61$\pm$08.87	& 90.84$\pm$05.66		& 73.39$\pm$11.25	& 86.43$\pm$06.49	& 81.60$\pm$07.80	& \bf{96.79$\pm$02.35}	\\
			8	& 92.23$\pm$05.24	& 91.87$\pm$05.22		& 70.44$\pm$09.99	& 81.48$\pm$05.65	& 78.57$\pm$05.82	& \bf{96.52$\pm$01.61}	\\
			9	& 87.52$\pm$03.63	& 91.15$\pm$06.22		& 66.02$\pm$05.31	& 82.31$\pm$04.65	& 81.57$\pm$04.89	& \bf{95.51$\pm$01.23}	\\
			10	& 90.61				& 95.77					& 74.18				& 82.16				& 82.36				& \bf{96.24}				\\	\hline
			Average		
			& 93.16				& 95.41					& 80.33				& 88.02				& 85.66				& \bf{97.67}				\\	\hline
			%---------------------& & & & & & & & 
		\end{tabular}
	}
	\label{tab_per_jaffe}
	%\end{tiny}
\end{table}

%\subsection{Evaluation Metrics}
%Three evaluation matrices are used in our experiment. The first measure is accuracy, ranging from 0 to 1. It measures the extent to which each cluster contains data points from the same class:
%\begin{equation}
%\label{eq_acc}
%\text{Accuracy} = \frac{\sum_{i=1}^{n} \delta (map(s_i),r_i)}{n},
%\end{equation}
%where $n$ is the total number of data points, $s_i$ and $r_i$ represent the predicted and ground-truth labels of the data point $x_i$, $map(s_i)$ is a mapping such that \eqref{eq_acc} can be maximized by permuting each cluster label $r_i$ to the equivalent label from the data set. The second measure, normalized mutual information (NMI), measures the quality of the clusters, which is defined as
%\begin{equation}
%\text{NMI} = \frac{\sum_{i=1}^{N}\sum_{j=1}^{N} n_{i,j} \log \frac{n_{i,j}}{n_i \hat{n_j}} }{\sqrt{ (\sum_{i=1}^{N} n_i \log \frac{n_i}{n} )(\sum_{j=1}^{N} \hat{n_j} \log \frac{\hat{n_j}}{n} ) }},
%\end{equation}
%where $n_i$ and $\hat{n_j}$ are the sizes of classes $i$ and $j$, $n_{i,j}$ is the their intersection size, and $N$ is the number of clusters. The third measure, purity, measures the extent to which each cluster contains samples from primarily the same class. It is defined as
%\begin{equation}
%\text{Purity} = \frac{1}{n}\sum_{i=1}^{N} \max(n_i^j),
%\end{equation}
%where $n_i^j$ is the number of data points in the $j$-th cluster that belong to the $i$-th class. More details about these measures can be found in \cite{huang2014robust}.

\subsection{Evaluation Metrics}
Three evaluation metrics are used in our experiment. The first metric is accuracy, ranging from 0 to 1. It measures the extent to which each cluster contains data points from the same class. The second metric, normalized mutual information (NMI), measures the quality of the clusters. The third metric, purity, measures the extent to which each cluster contains samples from primarily the same class. More details can be found in \cite{huang2014robust}.

\subsection{Benchmark Data Sets}
\label{sec_exp_data}
%Four benchmark data sets are used in our experiments, including PIX, JAFFE, Alphadigit, and Semeion. We briefly describe these data sets as follows: 1) PIX \cite{hond1997distinctive} contains 100 gray scale images collected from 10 objects, which has size 100$\times$100 pixes. 2) JAFFE \cite{lyons1998japanese} collects 213 images of 10 Japanese female models posed 7 facial expressions. These images are rated on 6 motion adjectives by 60 Japanese subjects. 3) Alphadigit is a binary data set, which collects handwritten digits 0-9 and letters A-Z. Totally, there are 36 classes and 39 samples for each class. 4) Semeion collects 1,593 handwritten digits that are written by around 80 persons. These images were scanned and stretched into size 16 $\times$16. 

Five benchmark data sets are used in our experiments, including PIX, JAFFE, Alphadigit, Semeion, and Faces94. We briefly describe these data sets as follows: 
\begin{itemize}
\item PIX \cite{hond1997distinctive} contains 100 gray scale images collected from 10 objects, which has size 100$\times$100 pixes. 
\item JAFFE \cite{lyons1998japanese} collects 213 images of 10 Japanese female models posed 7 facial expressions. These images are rated on 6 motion adjectives by 60 Japanese subjects. 
\item Alphadigit is a binary data set, which collects handwritten digits 0-9 and letters A-Z. Totally, there are 36 classes and 39 samples for each class. 
\item Semeion collects 1,593 handwritten digits that are written by around 80 persons. These images were scanned and stretched into size 16 $\times$16. 
\item Faces94 contains images of 153 individuals, each of whom has 20 images of size 200$\times$180. 
\end{itemize}

\begin{table}[!tb] %[PIX10]
	%\Huge
	\centering
	\caption{Clustering Performance on PIX }
	\resizebox{0.48\textwidth}{!}{
		%\begin{tabular}{ |c|c|c|c|c|c||c|}
		\begin{tabular}{|p{1cm}<{\centering} | p{2cm}<{\centering} | p{2cm}<{\centering} | p{2cm}<{\centering} | p{2cm}<{\centering} | p{2cm}<{\centering} || p{2cm}<{\centering} |}
			\hline
			\multirow{2}{1cm}{\centering N}& \multicolumn{6}{c|}{Accuracy (\%)}\\	
			\cline{2-7}
			\multirow{2}{1cm}{}
			& WNMF							& RMNMF					& CNMF							& KNMF				& ONMF				& KLS-NMF				\\\hline
			2	& 94.00$\pm$10.22				& \bf{96.50$\pm$07.84}	& 96.50$\pm$06.26	& 94.50$\pm$10.39	& 89.00$\pm$12.20	& 94.50$\pm$10.39		\\
			3	& 96.00$\pm$05.84				& \bf{97.33$\pm$03.06}	& 96.33$\pm$04.83				& 96.00$\pm$05.84	& 82.67$\pm$21.36	& 96.00$\pm$05.84		\\
			4	& 92.75$\pm$07.77				& 96.50$\pm$04.44		& 88.00$\pm$12.68				& 89.75$\pm$13.36	& 83.25$\pm$14.24	& \bf{97.25$\pm$03.81}	\\
			5	& 86.40$\pm$12.75				& \bf{90.80$\pm$07.50}	& 82.20$\pm$09.21				& 86.00$\pm$09.57	& 82.80$\pm$09.10	& 88.60$\pm$11.16		\\
			6	& 85.00$\pm$11.63				& 89.00$\pm$08.72		& 77.50$\pm$09.24				& 86.33$\pm$09.84	& 78.50$\pm$10.93	& \bf{90.17$\pm$09.51}	\\
			7	& 86.43$\pm$08.97				& 87.14$\pm$07.85		& 81.57$\pm$08.48				& 89.29$\pm$06.50	& 79.14$\pm$08.52	& \bf{92.00$\pm$06.32}	\\
			8	& 80.88$\pm$04.04				& 82.37$\pm$05.38		& 78.50$\pm$04.56				& 83.25$\pm$08.60	& 81.25$\pm$06.85	& \bf{91.00$\pm$01.84}	\\
			9	& 88.22$\pm$05.06				& 87.00$\pm$06.83		& 73.89$\pm$04.39				& 82.78$\pm$03.93	& 79.33$\pm$07.84	& \bf{91.00$\pm$04.81}	\\
			10	& 74.00							& 81.00					& 80.00							& 69.00				& \bf{89.00}		& \bf{89.00}			\\	\hline
			Average	
			& 87.08							& 89.74					& 83.83							& 86.32				& 82.77				& \bf{92.17}			\\	\hline
			%---------------------
			\hline
			\multirow{2}{1cm}{\centering N}& \multicolumn{6}{c|}{NMI (\%)}\\	
			\cline{2-7}
			\multirow{2}{1cm}{}
			& WNMF				& RMNMF					& CNMF				& KNMF				& ONMF				& KLS-NMF				\\	\hline
			2	& 81.39$\pm$28.27	& \bf{88.28$\pm$22.45}	& 87.45$\pm$21.45	& 83.81$\pm$28.77	& 67.62$\pm$32.75	& 83.81$\pm$28.77		\\
			3	& 89.87$\pm$11.64	& \bf{92.32$\pm$08.}18	& 91.00$\pm$08.50	& 89.87$\pm$11.64	& 79.34$\pm$16.44	& 89.87$\pm$11.64		\\
			4	& 89.41$\pm$09.06	& 93.45$\pm$07.08		& 84.84$\pm$11.15	& 88.21$\pm$11.69	& 82.36$\pm$11.02	& \bf{94.67$\pm$05.78}	\\
			5	& 87.90$\pm$09.29	& 88.04$\pm$07.35		& 79.03$\pm$09.09	& 83.93$\pm$07.46	& 84.46$\pm$04.78	& \bf{88.84$\pm$07.68}	\\
			6	& 86.02$\pm$08.30	& 87.05$\pm$07.23		& 75.43$\pm$08.46	& 87.75$\pm$06.03	& 81.94$\pm$07.25	& \bf{89.98$\pm$06.53}	\\
			7	& 88.64$\pm$04.80	& 87.06$\pm$07.05		& 82.52$\pm$05.68	& 88.39$\pm$05.28	& 83.33$\pm$05.91	& \bf{91.43$\pm$05.16}	\\
			8	& 85.16$\pm$02.09	& 83.54$\pm$04.15		& 80.95$\pm$03.37	& 87.80$\pm$04.71	& 84.36$\pm$05.58	& \bf{90.18$\pm$02.26}	\\
			9	& 89.22$\pm$01.67	& 87.89$\pm$04.59		& 78.50$\pm$03.96	& 85.59$\pm$01.77	& 84.60$\pm$04.81	& \bf{91.37$\pm$04.05}	\\
			10	& 83.91				& 86.02					& 82.97				& 80.90				& 89.31				& \bf{89.35}			\\	\hline
			Average	
			& 86.84				& 88.18					& 82.52				& 86.25				& 81.92				& \bf{89.94}			\\	\hline
			%--------------------
			\hline
			\multirow{2}{1cm}{\centering N}& \multicolumn{6}{c|}{Purity (\%)}\\	
			\cline{2-7}
			\multirow{2}{1cm}{}
			& WNMF				& RMNMF					& CNMF					& KNMF				& ONMF					& KLS-NMF						\\	\hline
			2	& 94.00$\pm$10.22	& \bf{96.50$\pm$07.84}	& \bf{96.50$\pm$06.26}	& 94.50$\pm$10.39	& 89.00$\pm$12.20		& 94.50$\pm$10.39				\\
			3	& 96.00$\pm$05.84	& \bf{97.33$\pm$03.06}	& 96.33$\pm$04.83		& 96.00$\pm$05.84	& 87.00$\pm$14.44		& 96.00$\pm$05.84				\\
			4	& 92.75$\pm$07.77	& 96.50$\pm$04.44		& 89.00$\pm$10.94		& 91.50$\pm$09.87	& 86.00$\pm$10.62		& \bf{97.25$\pm$03.81}	\\
			5	& 89.20$\pm$08.75	& \bf{91.40$\pm$06.11}	& 82.80$\pm$08.70		& 86.80$\pm$07.44	& 85.40$\pm$06.11		& 90.60$\pm$07.43				\\
			6	& 87.33$\pm$08.90	& {89.67$\pm$07.06}		& 79.00$\pm$08.47		& 88.50$\pm$07.00	& 82.17$\pm$08.32		& \bf{91.50$\pm$07.00}			\\
			7	& 88.71$\pm$06.19	& 88.57$\pm$06.02		& 83.71$\pm$06.50		& 89.71$\pm$05.34	& 82.14$\pm$06.50		& \bf{92.57$\pm$05.25}			\\
			8	& 83.88$\pm$02.66	& 85.00$\pm$03.82		& 81.12$\pm$03.30		& 86.25$\pm$06.01	& 82.87$\pm$06.18		& \bf{91.00$\pm$01.84}			\\
			9	& 89.44$\pm$03.24	& 88.22$\pm$05.14		& 77.22$\pm$04.57		& 84.78$\pm$02.03	& 82.22$\pm$05.93		& \bf{91.33$\pm$04.31}			\\
			10	& 79.00				& 85.00					& 83.00					& 74.00				& \underline{89.00}		& \bf{89.00}				\\	\hline
			Average	
			& 88.92				& 90.91					& 85.41					& 88.00				& 85.09					& \bf{92.64}				\\	\hline
			%---------------------& & & & & & & & 
		\end{tabular}
	}
	\label{tab_per_pix10}
	%\end{tiny}
\end{table}

\begin{table}[!tb] %[PIX10]
	%\Huge
	\centering
	\caption{Clustering Performance on Alphadigit }
	\resizebox{0.48\textwidth}{!}{
		%\begin{tabular}{ |c|c|c|c|c|c||c|}
		\begin{tabular}{|p{1cm}<{\centering} | p{2cm}<{\centering} | p{2cm}<{\centering} | p{2cm}<{\centering} | p{2cm}<{\centering} | p{2cm}<{\centering} || p{2cm}<{\centering} |}
			\hline
			\multirow{2}{1cm}{\centering N}& \multicolumn{6}{c|}{Accuracy (\%)}\\	
			\cline{2-7}
			\multirow{2}{1cm}{}
			& WNMF 				& RMNMF				& CNMF				& KNMF				& ONMF				& KLS-NMF						\\	\hline
			5	& 70.72$\pm$09.62	& 73.13$\pm$09.71	& 73.54$\pm$09.91	& 73.79$\pm$09.37	& 70.82$\pm$12.35	& \bf{81.38$\pm$12.89}			\\
			10	& 56.08$\pm$06.26	& 54.23$\pm$05.15	& 56.18$\pm$04.38	& 63.49$\pm$06.67	& 60.69$\pm$07.56	& \bf{65.46$\pm$07.03}			\\
			15	& 47.62$\pm$03.20	& 44.48$\pm$04.14	& 46.70$\pm$03.83	& 54.41$\pm$02.82	& 49.86$\pm$04.09	& \bf{54.50$\pm$03.89}	\\
			20	& 45.55$\pm$01.87	& 40.36$\pm$02.86	& 40.21$\pm$02.37	& 51.18$\pm$03.77	& 48.38$\pm$03.51	& \bf{52.88$\pm$03.56}			\\
			25	& 43.67$\pm$01.94	& 33.84$\pm$02.67	& 31.04$\pm$01.82	& 45.12$\pm$02.37	& 42.07$\pm$01.70	& \bf{48.61$\pm$02.94}	\\
			30	& 39.38$\pm$01.76	& 31.65$\pm$02.26	& 28.50$\pm$01.36	& 41.30$\pm$02.26	& 40.53$\pm$02.32	& \bf{45.88$\pm$02.84}	\\
			35	& 37.96$\pm$01.51	& 28.37$\pm$01.68	& 23.63$\pm$00.85	& 39.93$\pm$01.15	& 38.22$\pm$02.10	& \bf{44.29$\pm$01.60}	\\
			36	& 36.67				& 27.22				& 27.75				& 41.45				& 34.97				& \bf{41.74}				\\	\hline
			Average	
			& 47.21				& 41.66				& 40.94				& 51.33				& 48.19				& \bf{54.34}				\\	\hline
			%---------------------
			\hline
			\multirow{2}{1cm}{\centering N}& \multicolumn{6}{c|}{NMI (\%)}\\	
			\cline{2-7}
			\multirow{2}{1cm}{}
			& WNMF 				& RMNMF				& CNMF				& KNMF					& ONMF				& KLS-NMF						\\	\hline
			5	& 58.21$\pm$07.83	& 59.08$\pm$10.15	& 61.31$\pm$09.11	& 63.88$\pm$10.05		& 59.50$\pm$10.57	& \bf{69.82$\pm$11.61}			\\
			10	& 55.20$\pm$04.62	& 53.91$\pm$04.14	& 56.99$\pm$03.24	& 62.16$\pm$04.01		& 60.47$\pm$04.68	& \bf{64.13$\pm$04.87}			\\
			15	& 52.91$\pm$04.18	& 48.40$\pm$03.05	& 51.73$\pm$03.41	& \bf{58.49$\pm$03.16}	& 55.36$\pm$02.74	& \bf{58.42$\pm$02.02}	\\
			20	& 54.20$\pm$01.49	& 47.49$\pm$02.55	& 49.91$\pm$01.95	& 58.35$\pm$02.46		& 55.91$\pm$02.43	& \bf{61.55$\pm$02.44}			\\
			25	& 54.47$\pm$02.67	& 44.62$\pm$01.89	& 43.32$\pm$01.77	& 55.94$\pm$01.15		& 53.96$\pm$01.91	& \bf{59.78$\pm$01.84}			\\
			30	& 53.34$\pm$01.13	& 44.62$\pm$02.23	& 42.26$\pm$01.39	& 54.24$\pm$01.52		& 53.73$\pm$01.73	& \bf{58.65$\pm$01.49}			\\
			35	& 54.03$\pm$01.10	& 43.43$\pm$01.41	& 36.61$\pm$01.02	& 54.42$\pm$00.98		& 53.54$\pm$00.77	& \bf{58.81$\pm$00.94}			\\
			36	& 53.48				& 43.03				& 39.61				& 56.21					& 52.56				& \bf{56.64}					\\	\hline
			Average	
			& 54.48				& 48.07				& 47.72				& 57.96					& 55.64				& \bf{60.98}					\\	\hline
			%---------------------
			\hline
			\multirow{2}{1cm}{\centering N}& \multicolumn{6}{c|}{Purity (\%)}\\	
			\cline{2-7}
			\multirow{2}{1cm}{}
			& WNMF 				& RMNMF				& CNMF				& KNMF					& ONMF				& KLS-NMF						\\	\hline
			5	& 72.97$\pm$07.60	& 74.72$\pm$08.01	& 75.18$\pm$08.29	& 75.85$\pm$07.86		& 72.92$\pm$10.62	& \bf{82.36$\pm$10.84}			\\
			10	& 58.28$\pm$05.78	& 57.54$\pm$05.34	& 59.92$\pm$04.08	& 66.62$\pm$05.80		& 63.59$\pm$07.16	& \bf{67.87$\pm$06.47}			\\
			15	& 51.03$\pm$03.18	& 46.99$\pm$03.37	& 49.86$\pm$03.78	& \bf{59.70$\pm$03.07}	& 53.09$\pm$03.36	& 56.91$\pm$03.27				\\
			20	& 48.85$\pm$02.35	& 43.18$\pm$02.98	& 43.42$\pm$01.98	& 54.74$\pm$02.88		& 51.42$\pm$03.20	& \bf{56.29$\pm$03.48}			\\
			25	& 46.27$\pm$02.18	& 36.19$\pm$02.72	& 33.82$\pm$01.95	& 49.26$\pm$02.15		& 45.14$\pm$01.74	& \bf{51.63$\pm$02.80}	\\
			30	& 42.53$\pm$01.64	& 33.68$\pm$02.50	& 30.46$\pm$01.27	& 45.48$\pm$01.87		& 43.63$\pm$02.01	& \bf{49.52$\pm$02.78}			\\
			35	& 41.00$\pm$01.32	& 30.13$\pm$01.42	& 25.46$\pm$00.94	& 43.49$\pm$00.88		& 41.17$\pm$01.58	& \bf{47.41$\pm$01.}54	\\
			36	& 39.51				& 29.62				& 29.26				& \bf{45.37}			& 38.68				& 44.73							\\	\hline
			Average	
			& 50.05				& 44.01				& 43.42				& 54.84					& 51.21				& \bf{60.98}					\\	\hline
			%---------------------& & & & & & & & 
		\end{tabular}
	}
	\label{tab_per_alphadigit}
	%\end{tiny}
\end{table}

\subsection{Algorithms in Comparison}

To illustrate the effectiveness of KLS-NMF, we compare them with several state-of-the-art NMF methods, including weighted NMF (WNMF) \cite{kim2009weighted}, ONMF \cite{ding2006orthogonal}, CNMF \cite{ding2010convex}, Kernel NMF (KNMF) \cite{ding2010convex}, and RMNMF \cite{huang2014robust}. We briefly describe these methods as follows: 
\begin{itemize}
\item \textbf{WNMF}. It extends the results of the original NMF to a weighted case.
\item \textbf{ONMF}. It has different variants that imposes orthogonality constraint on different factor matrices. In our experiment, we adopt the matrix tri-factorization model that imposes orthogonal constraints on the left and right factor matrices. 
\item \textbf{CNMF}. It restricts the learned basis to lie within the column space of the input data, such that the basis vectors can be represented as a convex combination of the inputs. 
\item \textbf{KNMF}. Based on CNMF, KNMF exploits latent nonlinear structures of the data in kernel space. In our experiment, we use rbf kernel with radius ranging in the set $\mathcal{S} = \{0.001,0.01,0.1,1,10,100,1000\}$. 
\item \textbf{RMNMF}. It relaxes the data and basis matrix to have mixed signs, and adopts robust $\ell_{2,1}$ norm to measure the fitting errors. Moreover, nonlinear structures of the data are exploited on manifold. We use the binary weighting to construct the graph Laplacian, with the default 5 neighbors selected. The regularization parameter is also selected from $\mathcal{S}$. 
\item \textbf{KLS-NMF}. To be consistent with KNMF, we use rbf kernel with the same range and radius in $\mathcal{S}$. Moreover, we select the parameter $\lambda$ from the set $\mathcal{S}$. 
\end{itemize}

\begin{figure}[!tb]
	\centering
\includegraphics[width=0.95\columnwidth]{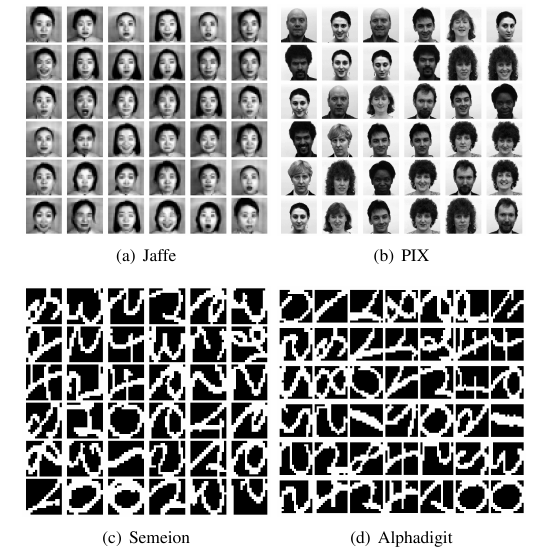}
	\caption{ Examples selected images from Jaffe, PIX, Semeion, and Alphadigit data sets. }
	\label{fig_images}
\end{figure}

\subsection{Clustering Performance}
\label{sec_exp_detail}
In this subsection, we evaluate the algorithms in comparison by conducting experiments on PIX, Jaffe, Semeion, and Alphadigit data sets. For purpose of illustration, we visually show some examples of these data sets in \cref{fig_images}. For a given data set, we denote the total number of clusters by $\bar{N}$, e.g., $\bar{N}=36$ in Alphadigit data. To better investigate the clustering performance of different methods on this data, we randomly select subsets of this data to conduct more detailed experiments. In particular, we randomly select a subset with $N$ out of $\bar{N}$ classes to conduct experiments. It is noted that for a specific $N$ value, there are $\bar{N}!/(\bar{N}-N)!N!$ different combinations of classes, i.e., subsets, from which we randomly chose 10. Experiments are conducted on the selected 10 subsets and the best average performance is reported by enumerating all possible combinations of parameters. This strategy applies to each data set and each algorithm. We test different $N$ values such that subsets of different sizes are tested for better comparison. We present the experimental results in \cref{tab_per_alphadigit,tab_per_jaffe,tab_per_pix10,tab_per_semeion}, with $N$ values being used in our experiments. In each table, three subtables are given corresponding to three evaluating measures, respectively. The best performance is bold-faced.

From \cref{tab_per_alphadigit,tab_per_jaffe,tab_per_pix10,tab_per_semeion}, it is observed that the proposed model has the best performance among all algorithms in comparison. In particular, KLS-NMF achieves the best performance in almost all cases and the improvements over other methods are significant. For example, on JAFFE data with large $N$ value, KLS-NMF can improve the performance by at about 8\% in all three measures. Generally, the proposed method can improve the average performance by around 3-6\% compared with the best competing method. It is noted that the best among compared methods varies depending on data, whereas the proposed method shows stability on all data sets. For example, RMNMF has some of the best results on PIX, but its performance on other data sets are less competitive. Moreover, the improvements of KLS-NMF over competing methods suggests that learning local similarity indeed provides advantages in clustering.

\subsection{Clustering Performance on Larger Data}
In the above subsection, we have evaluated the proposed method on some widely used benchmark data sets. Among them, 3 out of 4 data sets have up to 10 classes included in the experiment. In this subsection, we aim at testifying the capability of our method in handling larger data. To further testify how the proposed method performs on larger data, we conduct experiments on Faces94 data set. We use images of the males, where images of up to 113 inviduals are used in the experiment, which is fairly large for this test. For purpose of clearer illustration, we show some examples of this data in \cref{fig_faces94}. We follow the same settings as in above subsection and report the results in \cref{tab_per_faces94}. It is observed that the proposed method outperforms state-of-the-art algorithms in comparison with significant improvement. This observation, again, ensures the effectiveness of the proposed method and implies its potential to be used in real world applications.

\begin{table}[!tb] %[Jaffe]
	%\Huge
	\centering
	\caption{Clustering Performance on Faces94 }
	\resizebox{0.48\textwidth}{!}{
		%\begin{tabular}{ |c|c|c|c|c|c||c|}
		\begin{tabular}{|p{1cm}<{\centering} | p{2cm}<{\centering} | p{2cm}<{\centering} | p{2cm}<{\centering} | p{2cm}<{\centering} | p{2cm}<{\centering} || p{2cm}<{\centering} |}
			\hline
			\multirow{2}{1cm}{\centering N}& \multicolumn{6}{c|}{Accuracy (\%)}\\	
			\cline{2-7}
			\multirow{2}{1cm}{}
			& WNMF					& RMNMF				& CNMF				& KNMF 				& ONMF				& KLS-NMF						\\ \hline
			10	& 90.85$\pm$07.15 	& 86.80$\pm$08.89	& 82.25$\pm$06.77	& 87.85$\pm$07.77	& 79.75$\pm$09.66	& \textbf{100.0$\pm$00.00}			\\
			20	& 83.98$\pm$03.24 	& 81.47$\pm$03.13	& 79.78$\pm$05.39	& 78.90$\pm$04.99	& 76.00$\pm$03.42	& \textbf{88.89$\pm$03.47}			\\
			30	& 80.25$\pm$05.28 	& 82.08$\pm$05.59	& 75.78$\pm$04.21	& 72.15$\pm$04.50	& 72.12$\pm$03.34	& \textbf{83.97$\pm$04.12}			\\
			40	& 76.48$\pm$02.14 	& 76.40$\pm$02.98	& 72.59$\pm$04.58	& 70.36$\pm$04.68	& 69.74$\pm$02.59	& \textbf{83.47$\pm$02.62}			\\
			50	& 76.74$\pm$02.73 	& 76.37$\pm$02.75	& 72.45$\pm$02.36	& 68.24$\pm$03.46	& 67.27$\pm$02.99	& \textbf{82.13$\pm$03.12}			\\
			60	& 73.55$\pm$02.91 	& 77.91$\pm$03.66	& 71.16$\pm$02.72	& 66.14$\pm$02.44	& 68.11$\pm$03.29	& \textbf{81.72$\pm$03.25}			\\
			70	& 73.64$\pm$02.37 	& 76.11$\pm$04.38	& 72.30$\pm$03.57	& 67.26$\pm$03.22	& 69.01$\pm$02.66	& \textbf{80.14$\pm$03.33}			\\
			80	& 74.29$\pm$02.42 	& 77.83$\pm$02.20	& 69.42$\pm$02.99	& 65.70$\pm$02.32	& 68.19$\pm$03.43	& \textbf{78.45$\pm$03.26}			\\
			90	& 72.13$\pm$03.38 	& 77.24$\pm$01.96	& 67.92$\pm$02.47	& 63.44$\pm$01.78	& 68.26$\pm$03.79	& \textbf{79.49$\pm$02.73}			\\
			100	& 72.28$\pm$02.74 	& 77.41$\pm$02.80	& 67.09$\pm$02.63	& 65.20$\pm$01.74	& 69.07$\pm$02.75	& \textbf{79.72$\pm$01.83}			\\
			110	& 70.55$\pm$02.30 	& 76.38$\pm$02.61	& 66.85$\pm$02.76	& 63.90$\pm$02.41	& 70.42$\pm$02.06	& \textbf{77.82$\pm$02.85}			\\
			113	& 73.98 			& 74.42				& 66.55				& 65.93				& 68.50				& \textbf{77.26}			\\	\hline
			Average	
			& 76.56				& 78.37				& 76.35				& 69.59				& 70.54				& \textbf{82.76}				\\	\hline
			%---------------------
			\hline
			\multirow{2}{1cm}{\centering N}& \multicolumn{6}{c|}{NMI (\%)}\\	
			\cline{2-7}
			\multirow{2}{1cm}{}
			& WNMF				& RMNMF					& CNMF				& KNMF				& ONMF				& KLS-NMF						\\ \hline
			10	& 95.22$\pm$03.71 	& 90.03$\pm$06.61	& 90.87$\pm$03.68	& 94.00$\pm$03.75	& 89.85$\pm$04.90	& \textbf{100.0$\pm$00.00}			\\
			20	& 92.43$\pm$01.47 	& 89.61$\pm$02.18	& 89.49$\pm$02.06	& 89.84$\pm$02.21	& 88.35$\pm$02.48	& \textbf{95.72$\pm$01.19}			\\
			30	& 91.98$\pm$02.46 	& 90.13$\pm$03.69	& 88.46$\pm$01.88	& 87.44$\pm$01.80	& 88.20$\pm$01.68	& \textbf{94.41$\pm$01.41}			\\
			40	& 91.11$\pm$01.71 	& 88.78$\pm$02.43	& 86.48$\pm$02.33	& 85.84$\pm$03.52	& 87.12$\pm$01.34	& \textbf{94.86$\pm$00.74}			\\
			50	& 91.42$\pm$01.22 	& 89.27$\pm$02.54	& 86.79$\pm$01.67	& 85.52$\pm$01.94	& 85.83$\pm$01.44	& \textbf{94.20$\pm$00.88}			\\
			60	& 90.34$\pm$01.67 	& 89.57$\pm$02.66	& 86.70$\pm$01.26	& 84.55$\pm$01.72	& 87.37$\pm$01.93	& \textbf{94.30$\pm$01.07}			\\
			70	& 90.95$\pm$00.98 	& 91.33$\pm$01.65	& 87.74$\pm$01.76	& 85.56$\pm$02.52	& 88.69$\pm$01.32	& \textbf{94.11$\pm$01.08}			\\
			80	& 91.42$\pm$01.22 	& 91.40$\pm$01.49	& 86.26$\pm$01.44	& 84.41$\pm$01.98	& 88.83$\pm$01.48	& \textbf{93.92$\pm$01.16}			\\
			90	& 90.25$\pm$01.40 	& 91.77$\pm$01.73	& 85.63$\pm$00.91	& 82.87$\pm$01.59	& 88.19$\pm$01.77	& \textbf{94.11$\pm$00.97}			\\
			100	& 90.69$\pm$01.13 	& 91.83$\pm$01.12	& 85.39$\pm$01.00	& 84.34$\pm$01.94	& 88.90$\pm$01.06	& \textbf{94.20$\pm$00.51}			\\
			110	& 89.87$\pm$01.28 	& 91.36$\pm$01.01	& 85.37$\pm$01.16	& 83.58$\pm$01.82	& 89.40$\pm$00.88	& \textbf{93.91$\pm$00.88}			\\
			113	& 90.69 			& 91.67				& 85.43				& 83.74				& 87.85				& \textbf{94.01}			\\	\hline
			Average	
			& 91.36				& 90.56				& 87.05				& 85.97				& 88.21				& \textbf{94.81}				\\	\hline
			%---------------------	
			\hline
			\multirow{2}{1cm}{\centering N}& \multicolumn{6}{c|}{Purity (\%)}\\	
			\cline{2-7}
			\multirow{2}{1cm}{}
			& WNMF				& RMNMF					& CNMF				& KNMF				& ONMF				& KLS-NMF						\\	\hline
			10	& 93.05$\pm$05.25 	& 87.95$\pm$08.06	& 86.45$\pm$05.13	& 90.30$\pm$06.00	& 84.85$\pm$07.22	& \textbf{100.0$\pm$00.00}			\\
			20	& 87.50$\pm$02.01 	& 83.85$\pm$02.76	& 83.52$\pm$03.67	& 83.03$\pm$03.76	& 80.18$\pm$02.91	& \textbf{88.90$\pm$03.47}			\\
			30	& 84.95$\pm$03.57 	& 84.68$\pm$04.69	& 80.05$\pm$03.16	& 78.22$\pm$03.31	& 77.30$\pm$02.44	& \textbf{83.97$\pm$04.12}			\\
			40	& 82.20$\pm$01.94 	& 80.53$\pm$02.48	& 77.41$\pm$03.62	& 76.14$\pm$03.33	& 75.34$\pm$01.88	& \textbf{83.47$\pm$02.62}			\\
			50	& 82.00$\pm$02.15 	& 80.78$\pm$02.53	& 76.87$\pm$02.41	& 74.51$\pm$02.63	& 73.32$\pm$02.28	& \textbf{82.13$\pm$03.12}			\\
			60	& 79.53$\pm$02.44 	& 81.88$\pm$03.11	& 75.62$\pm$02.12	& 72.68$\pm$02.02	& 74.72$\pm$02.33	& \textbf{81.72$\pm$03.25}			\\
			70	& 79.94$\pm$01.98 	& 80.83$\pm$03.54	& 76.40$\pm$03.02	& 73.87$\pm$02.48	& 74.99$\pm$02.21	& \textbf{80.14$\pm$03.33}			\\
			80	& 80.31$\pm$01.73 	& 82.24$\pm$01.88	& 73.82$\pm$02.42	& 72.26$\pm$01.76	& 74.62$\pm$02.70	& \textbf{78.45$\pm$03.26}			\\
			90	& 78.94$\pm$02.60 	& 82.17$\pm$01.69	& 72.32$\pm$01.84	& 70.96$\pm$01.68	& 73.99$\pm$03.03	& \textbf{79.49$\pm$02.73}			\\
			100	& 78.55$\pm$01.94 	& 82.00$\pm$02.26	& 71.20$\pm$02.18	& 72.03$\pm$01.54	& 75.07$\pm$02.48	& \textbf{79.72$\pm$01.83}			\\
			110	& 77.50$\pm$01.75 	& 81.00$\pm$02.17	& 71.19$\pm$02.10	& 70.78$\pm$01.97	& 75.94$\pm$01.73	& \textbf{83.15$\pm$01.97}			\\
			113	& 79.60			 	& 79.73				& 71.42				& 72.48				& 73.05				& \textbf{82.74}			\\	\hline
			Average	
			& 82.01					& 82.30				& 76.35				& 75.60				& 76.11				& \textbf{83.66}				\\	\hline
			%---------------------& & & & & & & & 
		\end{tabular}
	}
	\label{tab_per_faces94}
	%\end{tiny}
\end{table}

\begin{figure}[!tb]
\centering
\includegraphics[width=1.0\columnwidth]{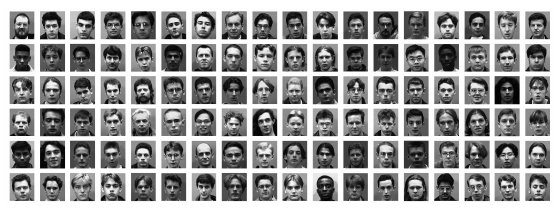}
\caption{ Examples selected images from Faces94 data. }
\label{fig_faces94}
\end{figure}

\subsection{Convergence}
In \cref{sec_proof}, we have provided theoretical analysis on the convergence of the proposed optimization strategy. To experimentally verify this, in this subsection, we will show some empirical examples. On Yale, PIX, Alphadigit, and Semeion data, we randomly choose 10 subsets. Without loss of generality, we use the 10 subsets with the smallest $N$ values as used in \cref{tab_per_alphadigit,tab_per_pix10,tab_per_semeion,tab_per_jaffe}. For all these subsets, we fix the parameter $\lambda = 0.001$ and set 1 for the radius of rbf kernel. 

In \cref{fig_W_kls}, we show how the difference of two consecutive $W_t$'s changes with respect to iteration number $t$ on the above selected subsets. Similarly, we show the distance sequence of two consecutive $G_t$'s in \cref{fig_G_kls}. It is seen that both $\{W_t\}$ and $\{G_t\}$ sequences can converge within a small number of iterations, which verifies the effectiveness and correctness of the optimization scheme.

\begin{figure}[!tb]
\centering
\includegraphics[width=0.48\textwidth]{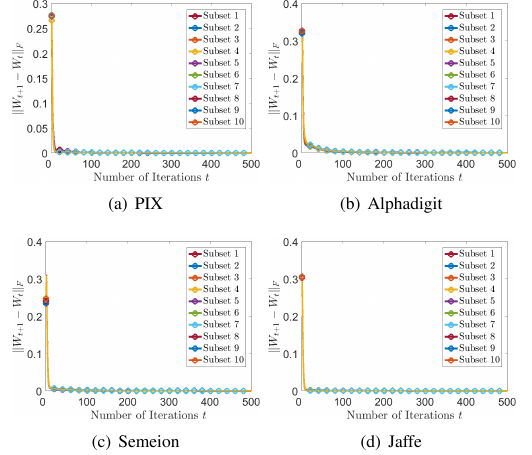}
%\subfigure[PIX]{\includegraphics[width=0.22\textwidth]{figures/fig_fval/fig_PIX_W.png}}
%\subfigure[Alphadigit]{\includegraphics[width=0.22\textwidth]{figures/fig_fval/fig_Alphadigit_W.png}}

%\subfigure[Semeion]{\includegraphics[width=0.22\textwidth]{figures/fig_fval/fig_Semeion_W.png}}
%\subfigure[Jaffe]{\includegraphics[width=0.22\textwidth]{figures/fig_fval/fig_Jaffe_W.png}}
%
\caption{ Example of the difference between consecutive $W_t$'s by KLS-NMF on PIX, Alphadigit, Semeion, and Jaffe. }
\label{fig_W_kls}
\end{figure}

\begin{figure}[!tb]
\centering
\includegraphics[width=0.48\textwidth]{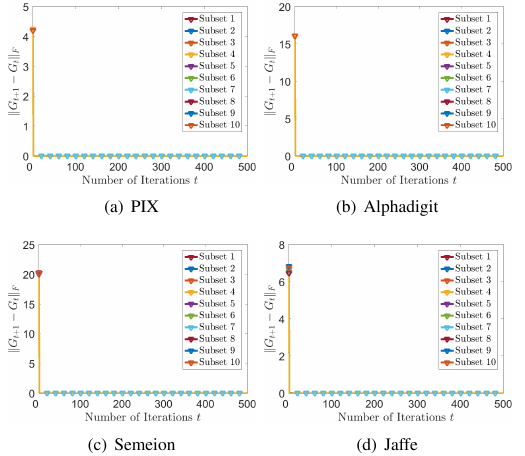}
%\subfigure[PIX]{\includegraphics[width=0.22\textwidth]{figures/fig_fval/fig_PIX_G.png}}
%\subfigure[Alphadigit]{\includegraphics[width=0.22\textwidth]{figures/fig_fval/fig_Alphadigit_G.png}}

%\subfigure[Semeion]{\includegraphics[width=0.22\textwidth]{figures/fig_fval/fig_Semeion_G.png}}
%\subfigure[Jaffe]{\includegraphics[width=0.22\textwidth]{figures/fig_fval/fig_Jaffe_G.png}}
%
\caption{ Example of the difference between consecutive $G_t$'s by KLS-NMF on PIX, Alphadigit, Semeion, and Jaffe. }
\label{fig_G_kls}
\end{figure}

Moreover, to further experimentally verify the convergence of objective value, we show some results in \cref{fig_fval_kls}. It is seen that the objective function indeed decreases its value with the updating rules on all these subsets. It is observed that the objective value sequences tend to converge within about 100 iterations, which verifies the fast convergence and effectiveness of the proposed method. In addition to the theoretical guarantees, these empirical observations indeed further strengthen the applicability of our method in real world problems. 

\begin{figure}[!tb]
\centering

\includegraphics[width=0.48\textwidth]{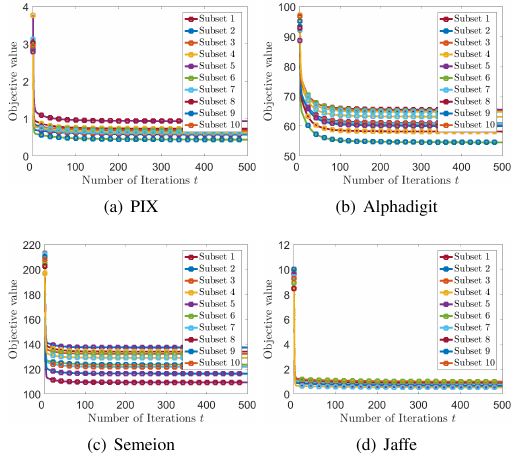}
%\subfigure[PIX]{\includegraphics[width=0.22\textwidth]{figures/fig_fval/fig_PIX_fval.png}}
%\subfigure[Alphadigit]{\includegraphics[width=0.22\textwidth]{figures/fig_fval/fig_Alphadigit_fval.png}}
%
%\subfigure[Semeion]{\includegraphics[width=0.22\textwidth]{figures/fig_fval/fig_Semeion_fval.png}}
%\subfigure[Jaffe]{\includegraphics[width=0.22\textwidth]{figures/fig_fval/fig_Jaffe_fval.png}}
%
\caption{ Example of objective value sequences by KLS-NMF on PIX, Alphadigit, Semeion, and Jaffe. }
\label{fig_fval_kls}
\end{figure}

\begin{figure}[!tb]
\centering
\includegraphics[width=0.48\textwidth]{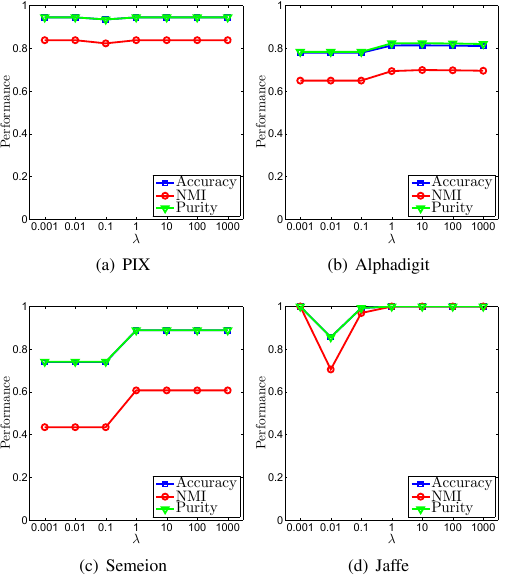}
%
%\subfigure[PIX]{\includegraphics[width=0.22\textwidth]{figures/fig_plots/fig_PIX10_paras_kernel.eps}}
%\subfigure[Alphadigit]{\includegraphics[width=0.22\textwidth]{figures/fig_plots/fig_Alphadigit_paras_kernel.eps}}
%
%\subfigure[Semeion]{\includegraphics[width=0.22\textwidth]{figures/fig_plots/fig_Semeion_paras_kernel.eps}}
%\subfigure[Jaffe]{\includegraphics[width=0.22\textwidth]{figures/fig_plots/fig_Jaffe_paras_kernel.eps}}
%
\caption{ Performance variations of KLS-NMF in accuracy, NMI, and purity with respect to different $\lambda$ values on PIX, Alphadigit, Semeion, and Jaffe. }
\label{fig_para_kls}
\end{figure}

\subsection{Parameter Sensitivity}
%For unsupervised learning, how to determine optimal parameters is still an open problem and yet to be exploited in further research. In this subsection, we test KLS-NMF and GLS-NMF with different combinations of parameters and show how the parameters affect their final clustering performance. Without loss of generality, we use the 10 subsets with the smallest $N$ values as used in \cref{tab_per_alphadigit,tab_per_pix10,tab_per_semeion,tab_per_jaffe}. For KLS-NMF, we plot the performance with respect to different $\lambda$ values with the best kernel used in \cref{sec_exp_detail}. For GLS-NMF, we test all combinations of $\lambda$ and $\gamma$ and report the best performance. The results are shown in \cref{fig_para,fig_para_kls}, respectively. It is observed that both KLS-NMF and GLS-NMF are quite insensitive to variation of parameters and promising performance can be obtained with a wide range of parameter combinations. This insensitivity to parameter variations may reduce parameter tuning effort, affording ease of use of our models in practice.
%
For unsupervised learning, how to determine optimal parameters is still an open problem and yet to be exploited in further research. In this subsection, we test KLS-NMF with different $\lambda$ values and show how it affects the final clustering performance. Without loss of generality, we use the 10 subsets with the smallest $N$ values as used in \cref{tab_per_alphadigit,tab_per_pix10,tab_per_semeion,tab_per_jaffe}. We plot the performance versus $\lambda$ with the best kernel used in the experiment. It is observed that KLS-NMF is quite insensitive to variation of parameters and promising performance can be obtained with a wide range of parameter variation. This insensitivity to parameter variation may reduce parameter tuning effort, affording ease of use of our models in practice.
%For GLS-NMF, we test all combinations of $\lambda$ and $\gamma$ and report the best performance. The results are shown in \cref{fig_para,fig_para_kls}, respectively. It is observed that both KLS-NMF and GLS-NMF are quite insensitive to variation of parameters and promising performance can be obtained with a wide range of parameter combinations. This insensitivity to parameter variations may reduce parameter tuning effort, affording ease of use of our models in practice.

\section{Conclusion}
\label{sec_conclusion}
This paper proposes a novel NMF method, which simultaneously exploits global and local structures of the data to construct basis vectors and coefficient matrix. The learned basis and coefficients well preserve intrinsic geometrical structures of the data and thus are more representative. An orthogonality constraint enforced on the coefficient and the embedding of local similarity learning mutually ensure the uniqueness of the factorization and provide an immediate and improved clustering interpretation. Nonlinear variant is developed and efficient multiplicative update rules are derived with theoretical convergence guarantee. Extensive experimental results have verified the effectiveness of the proposed method.

\section*{Acknowledgment}

This work is supported by National Natural Foundation of China (NSFC) under Grants 61806106, 61802215, and 61806045, and Shandong Provincial Natural Science Foundation, China under Grants ZR2019QF009, ZR2019BF028, and ZR2019BF011.

%C. Chen and Q. Cheng are corresponding authors. This work is supported by National Science Foundation under grant IIS-1218712, National Natural Science Foundation of China under grant 61201392, Science and Technology Planning Project of Guangdong Province, China under grant 2017B090909004, Research Project Supported by Shanxi Scholarship Council of China under grant 2015-093, Fundamental Research Fund for the Central Universities of China under grant ZYGX2017KYQD177, and Foundation Program of Yuncheng University under grants SWSX201603 and YQ-2012020.

%C. Chen and Q. Cheng are corresponding authors. 
%
%This work is supported by National Natural Foundation of China under grants 61806106, 61806045, and 61802215, National Science Foundation under grant IIS-1218712, Fundamental Research Fund for the Central Universities of China under grant ZYGX2017KYQD177.

%, and Foundation Program of Yuncheng University under grants SWSX201603 and YQ-2012020.

\ifCLASSOPTIONcaptionsoff
  \newpage
\fi

%
%\end{thebibliography}

\bibliographystyle{IEEEtrans}
\bibliography{LocalNMF_TCYB_v1}

% Generated by IEEEtranS.bst, version: 1.13 (2008/09/30)
\begin{thebibliography}{10}
\providecommand{\url}[1]{#1}
\csname url@samestyle\endcsname
\providecommand{\newblock}{\relax}
\providecommand{\bibinfo}[2]{#2}
\providecommand{\BIBentrySTDinterwordspacing}{\spaceskip=0pt\relax}
\providecommand{\BIBentryALTinterwordstretchfactor}{4}
\providecommand{\BIBentryALTinterwordspacing}{\spaceskip=\fontdimen2\font plus
\BIBentryALTinterwordstretchfactor\fontdimen3\font minus
  \fontdimen4\font\relax}
\providecommand{\BIBforeignlanguage}[2]{{%
\expandafter\ifx\csname l@#1\endcsname\relax
\typeout{** WARNING: IEEEtranS.bst: No hyphenation pattern has been}%
\typeout{** loaded for the language `#1'. Using the pattern for}%
\typeout{** the default language instead.}%
\else
\language=\csname l@#1\endcsname
\fi
#2}}
\providecommand{\BIBdecl}{\relax}
\BIBdecl

\bibitem{Arora2012Computing}
\BIBentryALTinterwordspacing
S.~Arora, R.~Ge, R.~Kannan, and A.~Moitra, ``Computing a nonnegative matrix
  factorization---provably,'' \emph{SIAM Journal on Computing}, vol.~45, no.~4,
  pp. 1582--1611, 2016. [Online]. Available:
  \url{https://doi.org/10.1137/130913869}
\BIBentrySTDinterwordspacing

\bibitem{buciu2008nonnegative}
I.~Buciu, N.~Nikolaidis, and I.~Pitas, ``Nonnegative matrix factorization in
  polynomial feature space,'' \emph{IEEE Transactions on Neural Networks},
  vol.~19, no.~6, pp. 1090--1100, 2008.

\bibitem{cai2011graph}
D.~Cai, X.~He, J.~Han, and T.~S. Huang, ``Graph regularized nonnegative matrix
  factorization for data representation,'' \emph{IEEE Transactions on Pattern
  Analysis and Machine Intelligence}, vol.~33, no.~8, pp. 1548--1560, 2011.

\bibitem{cai2008non}
D.~Cai, X.~He, X.~Wu, and J.~Han, ``Non-negative matrix factorization on
  manifold,'' in \emph{Data Mining, 2008. ICDM'08. Eighth IEEE International
  Conference on}.\hskip 1em plus 0.5em minus 0.4em\relax IEEE, 2008, pp.
  63--72.

\bibitem{candes2011robust}
E.~J. Cand{\`e}s, X.~Li, Y.~Ma, and J.~Wright, ``Robust principal component
  analysis?'' \emph{Journal of the ACM (JACM)}, vol.~58, no.~3, p.~11, 2011.

\bibitem{Chen2007Integrating}
J.~Chen, J.~Ye, and Q.~Li, ``Integrating global and local structures: A least
  squares framework for dimensionality reduction,'' in \emph{IEEE Conference on
  Computer Vision and Pattern Recognition}, 2007, pp. 1--8.

\bibitem{chen2013local}
J.~Chen, Z.~Ma, and Y.~Liu, ``Local coordinates alignment with global
  preservation for dimensionality reduction,'' \emph{IEEE transactions on
  neural networks and learning systems}, vol.~24, no.~1, pp. 106--117, 2013.

\bibitem{chung1997spectral}
F.~R. Chung, \emph{Spectral graph theory}.\hskip 1em plus 0.5em minus
  0.4em\relax American Mathematical Soc., 1997, vol.~92.

\bibitem{deerwester1990indexing}
S.~C. Deerwester, S.~T. Dumais, T.~K. Landauer, G.~W. Furnas, and R.~A.
  Harshman, ``Indexing by latent semantic analysis,'' \emph{JAsIs}, vol.~41,
  no.~6, pp. 391--407, 1990.

\bibitem{dhillon2007weighted}
I.~S. Dhillon, Y.~Guan, and B.~Kulis, ``Weighted graph cuts without
  eigenvectors a multilevel approach,'' \emph{IEEE Transactions on Pattern
  Analysis and Machine Intelligence}, vol.~29, no.~11, pp. 1944--1957, 2007.

\bibitem{ding2005equivalence}
C.~Ding, X.~He, and H.~D. Simon, ``On the equivalence of nonnegative matrix
  factorization and spectral clustering,'' in \emph{Proceedings of the 2005
  SIAM International Conference on Data Mining}.\hskip 1em plus 0.5em minus
  0.4em\relax SIAM, 2005, pp. 606--610.

\bibitem{ding2006orthogonal}
C.~Ding, T.~Li, W.~Peng, and H.~Park, ``Orthogonal nonnegative matrix
  t-factorizations for clustering,'' in \emph{Proceedings of the 12th ACM
  SIGKDD international conference on Knowledge discovery and data
  mining}.\hskip 1em plus 0.5em minus 0.4em\relax ACM, 2006, pp. 126--135.

\bibitem{ding2010convex}
C.~H. Ding, T.~Li, and M.~I. Jordan, ``Convex and semi-nonnegative matrix
  factorizations,'' \emph{IEEE transactions on pattern analysis and machine
  intelligence}, vol.~32, no.~1, pp. 45--55, 2010.

\bibitem{duda2012pattern}
R.~O. Duda, P.~E. Hart, and D.~G. Stork, \emph{Pattern classification}.\hskip
  1em plus 0.5em minus 0.4em\relax John Wiley \& Sons, 2012.

\bibitem{elhamifar2013sparse}
E.~Elhamifar and R.~Vidal, ``Sparse subspace clustering: Algorithm, theory, and
  applications,'' \emph{IEEE Transactions on Pattern Analysis and Machine
  Intelligence}, vol.~35, no.~11, pp. 2765--2781, 2013.

\bibitem{hond1997distinctive}
D.~Hond and L.~Spacek, ``Distinctive descriptions for face processing.'' in
  \emph{BMVC}, no. 0.2, 1997, pp. 0--4.

\bibitem{huang2014robust}
J.~Huang, F.~Nie, H.~Huang, and C.~Ding, ``Robust manifold nonnegative matrix
  factorization,'' \emph{ACM Transactions on Knowledge Discovery from Data
  (TKDD)}, vol.~8, no.~3, p.~11, 2014.

\bibitem{Kang2015Learning}
C.~Kang, S.~Xiang, S.~Liao, C.~Xu, and C.~Pan, ``Learning consistent feature
  representation for cross-modal multimedia retrieval,'' \emph{IEEE
  Transactions on Multimedia}, vol.~17, no.~3, pp. 370--381, 2015.

\bibitem{kim2009weighted}
Y.-D. Kim and S.~Choi, ``Weighted nonnegative matrix factorization,'' in
  \emph{Acoustics, Speech and Signal Processing, 2009. ICASSP 2009. IEEE
  International Conference on}.\hskip 1em plus 0.5em minus 0.4em\relax IEEE,
  2009, pp. 1541--1544.

\bibitem{lee1999learning}
D.~D. Lee and H.~S. Seung, ``Learning the parts of objects by non-negative
  matrix factorization,'' \emph{Nature}, vol. 401, no. 6755, pp. 788--791,
  1999.

\bibitem{lee2001algorithms}
------, ``Algorithms for non-negative matrix factorization,'' in \emph{Advances
  in neural information processing systems}, 2001, pp. 556--562.

\bibitem{li2006relationships}
T.~Li and C.~Ding, ``The relationships among various nonnegative matrix
  factorization methods for clustering,'' in \emph{Data Mining, 2006. ICDM'06.
  Sixth International Conference on}.\hskip 1em plus 0.5em minus 0.4em\relax
  IEEE, 2006, pp. 362--371.

\bibitem{liu2013robust}
G.~Liu, Z.~Lin, S.~Yan, J.~Sun, Y.~Yu, and Y.~Ma, ``Robust recovery of subspace
  structures by low-rank representation,'' \emph{IEEE Transactions on Pattern
  Analysis and Machine Intelligence}, vol.~35, no.~1, pp. 171--184, 2013.

\bibitem{liu2014global}
X.~Liu, L.~Wang, J.~Zhang, J.~Yin, and H.~Liu, ``Global and local structure
  preservation for feature selection,'' \emph{IEEE Transactions on Neural
  Networks and Learning Systems}, vol.~25, no.~6, pp. 1083--1095, 2014.

\bibitem{logothetis1996visual}
N.~K. Logothetis and D.~L. Sheinberg, ``Visual object recognition,''
  \emph{Annual review of neuroscience}, vol.~19, no.~1, pp. 577--621, 1996.

\bibitem{lyons1998japanese}
M.~J. Lyons, S.~Akamatsu, M.~Kamachi, J.~Gyoba, and J.~Budynek, ``The japanese
  female facial expression (jaffe) database,'' 1998.

\bibitem{ng2002spectral}
A.~Y. Ng, M.~I. Jordan, and Y.~Weiss, ``On spectral clustering: Analysis and an
  algorithm,'' in \emph{Advances in neural information processing systems},
  2002, pp. 849--856.

\bibitem{nie2012low}
F.~Nie, H.~Huang, and C.~Ding, ``Low-rank matrix recovery via efficient
  schatten p-norm minimization,'' in \emph{Twenty-Sixth AAAI Conference on
  Artificial Intelligence}, 2012.

\bibitem{nie2014clustering}
F.~Nie, X.~Wang, and H.~Huang, ``Clustering and projected clustering with
  adaptive neighbors,'' in \emph{Proceedings of the 20th ACM SIGKDD
  international conference on Knowledge discovery and data mining}.\hskip 1em
  plus 0.5em minus 0.4em\relax ACM, 2014, pp. 977--986.

\bibitem{nie2016constrained}
F.~Nie, X.~Wang, M.~I. Jordan, and H.~Huang, ``The constrained laplacian rank
  algorithm for graph-based clustering,'' in \emph{Thirtieth AAAI Conference on
  Artificial Intelligence}, 2016.

\bibitem{nie2011spectral}
F.~Nie, Z.~Zeng, I.~W. Tsang, D.~Xu, and C.~Zhang, ``Spectral embedded
  clustering: A framework for in-sample and out-of-sample spectral
  clustering,'' \emph{IEEE Transactions on Neural Networks}, vol.~22, no.~11,
  pp. 1796--1808, 2011.

\bibitem{Ogawa2011Missing}
T.~Ogawa and M.~Haseyama, ``Missing image data reconstruction based on adaptive
  inverse projection via sparse representation,'' \emph{IEEE Transactions on
  Multimedia}, vol.~13, no.~5, pp. 974--992, 2011.

\bibitem{palmer1977hierarchical}
S.~E. Palmer, ``Hierarchical structure in perceptual representation,''
  \emph{Cognitive psychology}, vol.~9, no.~4, pp. 441--474, 1977.

\bibitem{Peng2017Subspace}
C.~Peng, Z.~Kang, and Q.~Cheng, ``Subspace clustering via variance regularized
  ridge regression,'' in \emph{IEEE Conference on Computer Vision and Pattern
  Recognition}, 2017, pp. 682--691.

\bibitem{peng2015subspace}
C.~Peng, Z.~Kang, H.~Li, and Q.~Cheng, ``Subspace clustering using
  log-determinant rank approximation,'' in \emph{Proceedings of the 21th ACM
  SIGKDD International Conference on Knowledge Discovery and Data
  Mining}.\hskip 1em plus 0.5em minus 0.4em\relax ACM, 2015, pp. 925--934.

\bibitem{shi2000normalized}
J.~Shi and J.~Malik, ``Normalized cuts and image segmentation,''
  \emph{Departmental Papers (CIS)}, p. 107, 2000.

\bibitem{tao2007general}
D.~Tao, X.~Li, X.~Wu, and S.~J. Maybank, ``General averaged divergence
  analysis,'' in \emph{Data Mining, 2007. ICDM 2007. Seventh IEEE International
  Conference on}.\hskip 1em plus 0.5em minus 0.4em\relax IEEE, 2007, pp.
  302--311.

\bibitem{vavasis2009complexity}
S.~A. Vavasis, ``On the complexity of nonnegative matrix factorization,''
  \emph{SIAM Journal on Optimization}, vol.~20, no.~3, pp. 1364--1377, 2009.

\bibitem{wachsmuth1994recognition}
E.~Wachsmuth, M.~Oram, and D.~Perrett, ``Recognition of objects and their
  component parts: responses of single units in the temporal cortex of the
  macaque,'' \emph{Cerebral Cortex}, vol.~4, no.~5, pp. 509--522, 1994.

\bibitem{Wang2008Locality}
H.~Wang, S.~Chen, Z.~Hu, and W.~Zheng, ``Locality-preserved maximum information
  projection,'' \emph{IEEE Transactions on Neural Networks}, vol.~19, no.~4,
  pp. 571--585, 2008.

\bibitem{wang2013nonnegative}
Y.-X. Wang and Y.-J. Zhang, ``Nonnegative matrix factorization: A comprehensive
  review,'' \emph{IEEE Transactions on Knowledge and Data Engineering},
  vol.~25, no.~6, pp. 1336--1353, 2013.

\bibitem{xu1996convergence}
L.~Xu and M.~I. Jordan, ``On convergence properties of the em algorithm for
  gaussian mixtures,'' \emph{Neural computation}, vol.~8, no.~1, pp. 129--151,
  1996.

\bibitem{yoo2010orthogonal}
J.~Yoo and S.~Choi, ``Orthogonal nonnegative matrix tri-factorization for
  co-clustering: Multiplicative updates on stiefel manifolds,''
  \emph{Information processing \& management}, vol.~46, no.~5, pp. 559--570,
  2010.

\bibitem{zass2005unifying}
R.~Zass and A.~Shashua, ``A unifying approach to hard and probabilistic
  clustering,'' in \emph{Computer Vision, 2005. ICCV 2005. Tenth IEEE
  International Conference on}, vol.~1.\hskip 1em plus 0.5em minus 0.4em\relax
  IEEE, 2005, pp. 294--301.

\bibitem{zhao2017robust}
N.~Zhao, L.~Zhang, B.~Du, Q.~Zhang, J.~You, and D.~Tao, ``Robust dual
  clustering with adaptive manifold regularization,'' \emph{IEEE Transactions
  on Knowledge and Data Engineering}, vol.~29, no.~11, pp. 2498--2509, Nov
  2017.

\bibitem{Zhen2015Spectral}
Y.~Zhen, Y.~Gao, D.~Y. Yeung, H.~Zha, and X.~Li, ``Spectral multimodal hashing
  and its application to multimedia retrieval,'' \emph{IEEE Transactions on
  Cybernetics}, vol.~46, no.~1, pp. 27--38, 2015.

\bibitem{zhu2018low}
X.~Zhu, S.~Zhang, Y.~Li, J.~Zhang, L.~Yang, and Y.~Fang, ``Low-rank sparse
  subspace for spectral clustering,'' \emph{IEEE Transactions on Knowledge and
  Data Engineering}, 2018.

\bibitem{Zhu2014Fast}
Z.~Zhu, F.~Guo, H.~Yu, and C.~Chen, ``Fast single image super-resolution via
  self-example learning and sparse representation,'' \emph{IEEE Transactions on
  Multimedia}, vol.~16, no.~8, pp. 2178--2190, 2014.

\end{thebibliography}

\vfill

\end{document}